\documentclass[twoside,11pt]{jmlr} 

\usepackage{times}
\usepackage{amsmath}
\usepackage{algorithm}
\usepackage{xcolor}
\usepackage{enumitem}
\usepackage{graphicx}
\newcommand{\CalO}{\mathcal O}
 \newcommand{\Real}{\mathbb R}

    \newcommand{\BA}{\textbf A}
     \newcommand{\BPhi}{\boldsymbol \Phi}
      \newcommand{\BPsi}{\boldsymbol \Psi}

    \newcommand{\BT}{\mathbf T}
     \newcommand{\Bt}{\mathbf t}
     \newcommand{\Bx}{\mathbf x}

      \newcommand{\Bd}{\mathbf d}

    \newcommand{\BZ}{\mathbf Z}
    \newcommand{\BFH}{\mathbf H}
    \newcommand{\BFM}{\mathbf M}
    \newcommand{\Btheta}{\boldsymbol \theta}
    \newcommand{\BTheta}{\boldsymbol \Theta}
    \newcommand{\CalL}{\mathcal L}
    
    \newcommand{\CalH}{\mathcal H}
    
    \newcommand{\CalK}{\mathcal K}
    
    \newcommand{\CalN}{\mathcal N}
    
    \newcommand{\CalW}{\mathcal W}
    \newcommand{\E}{\mathbb E}
    \newtheorem{condition}{Condition}

\allowdisplaybreaks

\begin{document}

\title[A General Theory for Kernel Packets]{Beyond State Space Representation: A General Theory for Kernel Packets }

 \author{
    \Name{Liang Ding} \Email{liang\_ding@fudan.edu.cn}\\
    \addr{School of Data Science, Fudan University, Shanghai, China}\\
    \Name{Rui Tuo} \Email{ruituo@tamu.edu}\\
    \addr{Wm Michael Barnes '64 Department of Industrial \& Systems Engineering\\
    Texas A\&M University, College Station, TX 77843, USA}\\
    \Name{Lu Zou} \Email{zoulu1990330@gmail.com}\\
    \addr{School of Management, Shenzhen Polytechnic University, Shenzhen, China}
    }

\maketitle

\begin{abstract}%
Gaussian process (GP) regression provides a flexible, nonparametric framework for probabilistic modeling, yet remains computationally demanding in large-scale applications. For one-dimensional data, state space (SS) models achieve linear-time inference by reformulating GPs as stochastic differential equations (SDEs). However, SS approaches are confined to gridded inputs and cannot handle multi-dimensional  scattered data.
We propose a new framework based on kernel packet (KP), which overcomes these limitations while retaining exactness and scalability. A KP is a compactly supported function defined as a linear combination of the GP covariance functions. In this article, we prove that KPs can be identified via the forward and backward SS representations. We also show that the KP approach enables exact inference with linear-time training and logarithmic or constant-time prediction, and extends naturally to multi-dimensional  gridded or scattered data without low-rank approximations. Numerical experiments on large-scale additive and product-form GPs with millions of samples demonstrate that KPs achieve exact, memory-efficient inference where SDE-based and low-rank GP methods fail.
 
\end{abstract}


\section{Introduction}


Gaussian process (GP) models provide a flexible, probabilistic, and nonparametric framework for interpolation, forecasting, and smoothing \citep{RasmussenWilliams06}. Despite their flexibility, GPs remain computationally demanding in large-scale applications due to their poor scalability with data size. Specifically, exact GP regression requires $\mathcal{O}(n^3)$ training time and $\mathcal{O}(n)$ prediction time for $n$ observations, which limits its use in large-scale datasets.  To address this limitation, various approaches have been developed to accelerate GP inference. One efficient line of research focuses on GPs governed by stochastic differential equations (SDEs), which admit equivalent state space (SS) representations \citep{solin2016stochastic}. In this framework, the GP is treated as the output of a  SDE driven by white noise, and exact inference can be achieved in linear time by solving an equivalent SS formulation. The SS formulation is particularly appealing because it requires only $\mathcal{O}(n)$ time  for exact computation \citep{hartikainen2010kalman,saatcci2012scalable,solin2016stochastic}. However, the SS approach has several inherent limitations. First, its original formulation requires one-dimensional input and predetermined prediction points, and therefore cannot be readily extended to multi-dimensional scattered data. Second, if the prediction point is not predetermined, it requires a prediction time of $\mathcal{O}(n)$, making it unsuitable for problems involving a large number of prediction points.



In this work, we propose a framework that extends the idea of KPs in  \cite{chen2022kernel}, which shows that linear combinations of Mat\'ern kernels  possess the compact-support property. We show that KPs can be obtained by combining both the forward and backward SS models. The compact support of KPs brings substantial computational advantages. While both  KP and SS approaches lead to $\CalO(n)$ time algorithms for the GP regression model, KP surpasses SS in log-likelihood  and predictive computational efficiency.  Unlike SS, which cannot improve the computational efficiency for evaluating the log-likelihood and requires $\CalO(n)$ time per prediction, KP reduces computation to $\CalO(n)$ for the log-likelihood and $\CalO(\log n)$ or $\CalO(1)$ for predictions. 

We also provide natural extensions of KPs to multi-dimensional problems with scattered inputs, which overcome a major structural constraint of the SS representation. Our algorithm constructs KPs directly from the SDE formulation and provides exact inference without resorting to low-rank approximations or variational methods.

Finally, we evaluate the proposed method on large-scale additive and product-form GPs with millions of training and test samples. The results demonstrate that KPs achieve exact, memory-efficient inference where existing SDE-based and low-rank GP methods fail due to memory or approximation limitations. To summarize,  we have the following three contributions: 
\begin{enumerate}
    \item  We develop an exact algorithm that generalizes SS models. Compared with SS methods, it retains linear-time training but achieves faster log-likelihood evaluation, prediction, and kernel–matrix computations by exploiting the compact support of KPs;
    \item  We extend kernel packets to multidimensional scattered-data settings, going beyond the one-dimensional SS model regime;
    \item  {We establish a unified framework that bridges SDE-based and kernel-based Gaussian process inference, and empirically demonstrate its exactness and scalability through  experiments with millions of samples and test points.}
\end{enumerate}

\subsection{Literature review}
\label{sec:review}
Scalable GP regression has been addressed through several approximation strategies.
Likelihood-based methods simplify the joint likelihood to reduce computational cost. Representative examples include pseudo-likelihoods \cite{varin2011overview,eidsvik2014estimation} and the Vecchia approximation \cite{stein2004approximating,katzfuss2021general}, both of which approximate the dependence structure among observations to achieve scalability.
Covariance tapering provides another approach by multiplying the covariance function with a compactly supported kernel, producing sparse covariance matrices that can be inverted efficiently \cite{furrer2006covariance,kaufman2008covariance,stein2013statistical}.
Random feature methods approximate kernels using stochastic basis functions \cite{rahimi2008random,le2013fastfood,hensman2017variational}.
Local approximations divide the input space into subregions and fit independent or weakly coupled GPs within each \cite{gramacy2015local,cole2021locally}.
Finally, low-rank approximations have been proposed from various perspectives, including discrete process convolutions \cite{higdon2002space}, fixed rank kriging \cite{cressie2008fixed,kang2011bayesian}, predictive processes \cite{banerjee2008gaussian,finley2009improving}, lattice kriging \cite{nychka2015multiresolution}, hierarchical matrices \cite{chen2023linear} and stochastic partial differential equations \cite{lindgren2011explicit},  among others. These approaches construct finite-dimensional representations of the underlying GP, typically using structured basis functions to reduce complexity. These approaches are generally applicable, but their computational efficiency is gained at the cost of accuracy. Another direction is to seek exact and scalable algorithms under specific covariance functions and experimental designs. When the design points are regular (i.e., equally-spaced) grids, Toeplitz methods can be applied to reduce the computational complexity \cite{wood1994simulation}. However, regular grids are too restrictive in computer experiment applications, and they are sub-optimal in terms of the prediction performance in multi-dimensional problems. A more powerful class of designs is the sparse grids. With these designs, \cite{plumlee2014fast} proposed an algorithm for the inference and prediction of GP models. Although this algorithm is faster than directly inverting the $n\times n$ covariance matrix, its training time complexity remains $\CalO(n^3 )$ under a fixed input dimension. When the GP can be represented as a SS model, Kalman filtering and smoothing can be applied to provide an efficient prediction algorithm \cite{hartikainen2010kalman,saatcci2012scalable,sarkka2013spatiotemporal,solin2016stochastic,loper2021general}. {But this approach has a major downside: the nature of Kalman filtering and smoothing requires specifying the input points where the algorithm is going to make predictions in the training process. This makes SS model difficult to apply in many applications where input points are not known in advance.}

\section{Preliminaries}
\label{sec:motivation}
A GP $y(t)\sim\CalN(\mu,K)$ is characterized by its mean function $\mu(t)$ and kernel function $K(t,t')=\E[y(t)y(t')]$. In preliminary, we assume that $y(t)$ is central, i.e., $\mu=0$ and observations are noiseless. Given observations $Y=[y(t_1),\ldots,y(t_n)]^\top$, by the fact that $y(t)$ is Gaussian distributed at any point $t$, it is straightforward to derive that  the posterior distribution of $y(t)$ for any untried $t$ is a multivariate normal distribution given by
\begin{eqnarray}\label{GPR}
    y(t)|Y\sim N(K(t,\mathbf{T})K^{-1}(\mathbf{T},\mathbf{T})Y,K(t,\mathbf{T})K^{-1}(\mathbf{T},\mathbf{T})K(\mathbf{T},t)),
\end{eqnarray}
 Note that (\ref{GPR}) involves inverting an $n\times n$ matrix, which prohibits the scalability of GP regression in its original form. We review two classes of methodologies to resolve this issue.

\subsection{State space models}
\label{sec:SS}
The state space model relies on a SDE representation of GP $y(t)$ as 
\begin{equation}
\label{eq:SODE}
\mathcal{L}[y] := y^{(m)}(t) + c_{m-1}(t)y^{(m-1)}(t) + \cdots + c_0(t)y(t) = W(t), \quad t \in [t_0, T],
\end{equation}
where $W(t)$ is a white noise process with unit spectral density. For our theorems and algorithms to be valid, we impose a mild condition on SDE \eqref{eq:SODE}:
 \begin{condition}
\label{condition:full_rank}
There exists $m$ linearly independent fundamental solutions $h_i$ to the operator $\CalL$ (i.e. 
$\CalL h_i=0$)
and  each $h_i$ is bounded.
\end{condition}

The SDE formulation is then further reformulated as the following Markov SS model
\begin{eqnarray}\label{SS}
\begin{aligned}
    \partial_t z(t)&=F(t)z(t)+LW(t) & \text{ (dynamic model)}\\
    y(t)&=H z(t) & \text{ (measurement model)}
\end{aligned}\quad ,
\end{eqnarray}
where $L=[0,\cdots,0,1]$, $H=[1,0,\cdots,0]\in\Real^m$ are vectors,  and $F(t)\in\Real^{m\times m}$ is a matrix-valued function of $t$,  respectively. Given a set of observations $\{y(t_i)\}_{i=1}^n$ and a predetermined prediction point $t^*$, we can insert $t^*$ into the sequence $\{t_i\}_{i=1}^n$ such that $t_i < t^* < t_{i+1}$, and sequentially compute the distribution of \eqref{SS} for $y(t_1), \ldots, y(t_i), y(t^*)$, which requires only $\mathcal{O}(n)$ time.

The Markov SS model representation of GP \eqref{eq:SODE} is not unique. We use GPs with Mat\'ern covariance functions as an example. It can be derived from the spectral density of a  Mat\'ern GP that it has the following SDE representation:
\begin{align}
\label{eq:SODE_matern}
    (\partial_t+\lambda)^m y(t)=W(t)=\sum_{j=0}^m{m\choose j}\lambda^{m-j}y^{(j)}(t)
\end{align}
The left-hand side and right-hand side of \eqref{eq:SODE_matern} yield two different Markov SS representations 
\begin{equation*}
      \partial_t z_1(t)=F_1z_1(t)+LW(t),\quad   \partial_tz_2(t)=F_2z_2(t)+LW(t)
\end{equation*}
with $z_1=[y,y^{(1)},\cdots,y^{(m-1)}]$, $z_2=[y,(\partial_t+\lambda)y,\cdots,(\partial_t+\lambda)^{m-1}y]$, and
\begin{equation*}
    F_1=-\begin{bmatrix}
0& -1 & 0 &\cdots & 0\\
& & \vdots & &\\
0&0&0& \cdots&-1\\
\lambda^m& m\lambda^{m-1}&{m\choose 2}\lambda^{m-2}&\cdots &m\lambda
\end{bmatrix},\ F_2=\begin{bmatrix}
-\lambda& 1 & 0 &\cdots & 0\\
0& -\lambda& 1  &\cdots & 0\\
& & \vdots & &\\
0&0&0&\cdots &-\lambda
\end{bmatrix}.
\end{equation*}
Note that any SDE can be reformulated in the form corresponding to $z_1(t)$; hence, it is refered as  the \emph{canonical} SS model. In this study, we do not restrict to a specific SS model representation, as our KP algorithm is applicable to \emph{any} such representation. 

{{In the following, we denote the covariance matrix of the SS model~\eqref{SS} by
\begin{equation}\label{eq:R_def}
R(t,\mu) = \mathbb{E}[z(t)z(\mu)],  \qquad R_{i,j}(t,\mu) = \mathbb{E}[z_i(t)z_j(\mu)]
\end{equation}
without committing to a particular SS representation. The kernel function associated to $y(t)$ is $K=R_{1,1}$. Derivation of kernel functions  from SS models is provided in Appendix \ref{sec:conversion}.}}


\subsection{Kernel functions asscoiate with SDEs}  
Since our study concerns efficient GP regression with genereal kernel $K$, a natural question arises: what types of kernel functions are associated with the SDE \eqref{eq:SODE}? As shown in \cite{solin2016stochastic,benavoli2016state}, a wide range of covariance functions—including Matérn kernels, spline kernels, and neural network kernels—can be represented through GPs in SDE form. We use two general kernel classes widely applied in physical and engineering problems to demonstrate that the range of kernels associated to SDEs is broad.

The first class is the well-known CARMA process \citep{brockwell2001continuous}, which corresponds to  stationary SDEs where the coefficients $\{c_j(t)=c_j\}_{j=0}^{m-1}$ in \eqref{eq:SODE} are time-invariant. Kernel functions associated with stationary SDEs are isotropic, i.e., $K(t,t')=K(t-t')$. So by multiplying both sides of \eqref{eq:SODE} by $y(t')$, taking expectations, and introducing the change of variable $\mu=t-t'$, we can get the following ODE representation of $K$:
\begin{equation}
    \label{eq:ODE_K}
    K^{(m)}(\mu)+c_{m-1}K^{(m-1)}(\mu)+\cdots+c_0K(\mu)=0.
\end{equation}
From  direct calculations, we can derive that $K$ is of the form $ K(\mu)=\sum_{j=0}^m A_j(\mu)e^{-\omega_j|\mu|}$ where $A_j(\mu)$ is proportional to one of $1$, $|\mu|^j$, $\cos(\alpha_j |\mu|)$, or $\sin(\beta_j |\mu|)$, depending on the roots of the characteristic polynomial of \eqref{eq:ODE_K}. A special case is the Mat\'ern kernel, where {$c_j= \lambda^j{m\choose j}$} \citep{stein1999interpolation}. \cite[p. 326]{papoulis1965random} also gave three examples for the case $m=2$.

The second class is from convolution of kernels, which corresponds to SDE of the form
\begin{equation}
 \label{eq:SODE_convolution}
     \begin{aligned}
        & \mathcal{L}_1[y_1]:=y_1^{(m_1)}(t)+c_{1,m_1-1}(t)y_1^{(m_1-1)}(t)+\cdots+c_{1,0}(t)y_1(t)=y_2(t)\\
        & \mathcal{L}_2[y_2]:=y_2^{(m_2)}(t)+c_{2,m_2-1}(t)y_2^{(m_2-1)}(t)+\cdots+c_{2,0}(t)y_2(t)=W(t)
     \end{aligned},\quad t\in[t_0,T],
 \end{equation}
where the kernel $K_2$ of $y_2$ is known, while $y_1$ is constrained by the physical conditions imposed through the operator $\mathcal{L}_1$. Then we have $K_1(t,t')=\int_{t_0}^{T} G(t,\mu)K_2(\mu,\mu')G(\mu',t')d\mu d\mu'$, where $G$ is the Green's function of $\mathcal{L}_1$. This method enables the construction of physics-informed GPs with closed-form kernels that are readily obtainable.  \cite{ding2025bdrymat,dalton2024boundary} employ this convolution method to construct kernels that satisfy general boundary conditions.

\subsection{Definition of Kernel packets}
The method of kernel packets looks for a different sparse representation of $K(\mathbf{T},\mathbf{T})$ in (\ref{GPR}). For Mat\'ern kernels with smoothness $\nu=p-1/2$ and $p\in\mathbb{N}_+$, \cite{chen2022kernel} proved that $K(\mathbf{T},\mathbf{T})=A^{-1}\Phi$, where $A$ and $\Phi$ are banded matrices  with bandwidth $p$ and $p-1$, respectively. To explain this representation, we need the definition of KPs. Denote $a\wedge b:=\min(a,b)$ and $a\vee b:=\max(a,b)$. Below is a rephrased version of the original definition in \cite{chen2022kernel}.

\begin{definition}
    Given $p,n\in\mathbb{N}_+$ with $n>2p+1$, positive definite function $K(\cdot,\cdot)$, and input points $t_1<\ldots<t_n\in(t_0,T)$, a set of functions $\{\phi_1(\cdot),\ldots,\phi_n(\cdot)\}$ is called a kernel packet system with degree $2p+1$ if
    \begin{enumerate}
        \item $\phi_j=\sum_{k=(j-p)\vee 0}^{(j+p)\wedge n} a_{jk} K(\cdot,t_k)$ for some not-all-zero constants $a_{jk}$'s.
        \item For $j=1,\ldots,p$, $\phi_j(t)=0$ whenever $t\geq t_{j+p}$. These $\phi_j$'s are called left-sided KPs.
        \item For $j=p+1,\ldots,n-p-1$, $\phi_j(t)=0$ for any $t\not\in(t_{j-p},t_{j+p})$. These $\phi_j$'s are called KPs.
        \item For $j=n-p,\ldots,n$, $\phi_j(t)=0$ for any $t\leq t_{j-p}$. These $\phi_j$'s are called right-sided KPs.
    \end{enumerate}
    A kernel packet system is called a kernel packet basis if the functions are linearly independent.
\end{definition}

In other words, a KP basis is a linear transform of $\{K(\cdot,t_1),\ldots,K(\cdot,t_n)\}$ and is (mostly) compactly supported. This leads to the aforementioned sparse representation and an $O(n)$ time GP regression algorithm. See Figure \ref{fig:KP_basis} for an illustration of KP bases.

\begin{figure}
\centering
\includegraphics[width=0.35\linewidth]{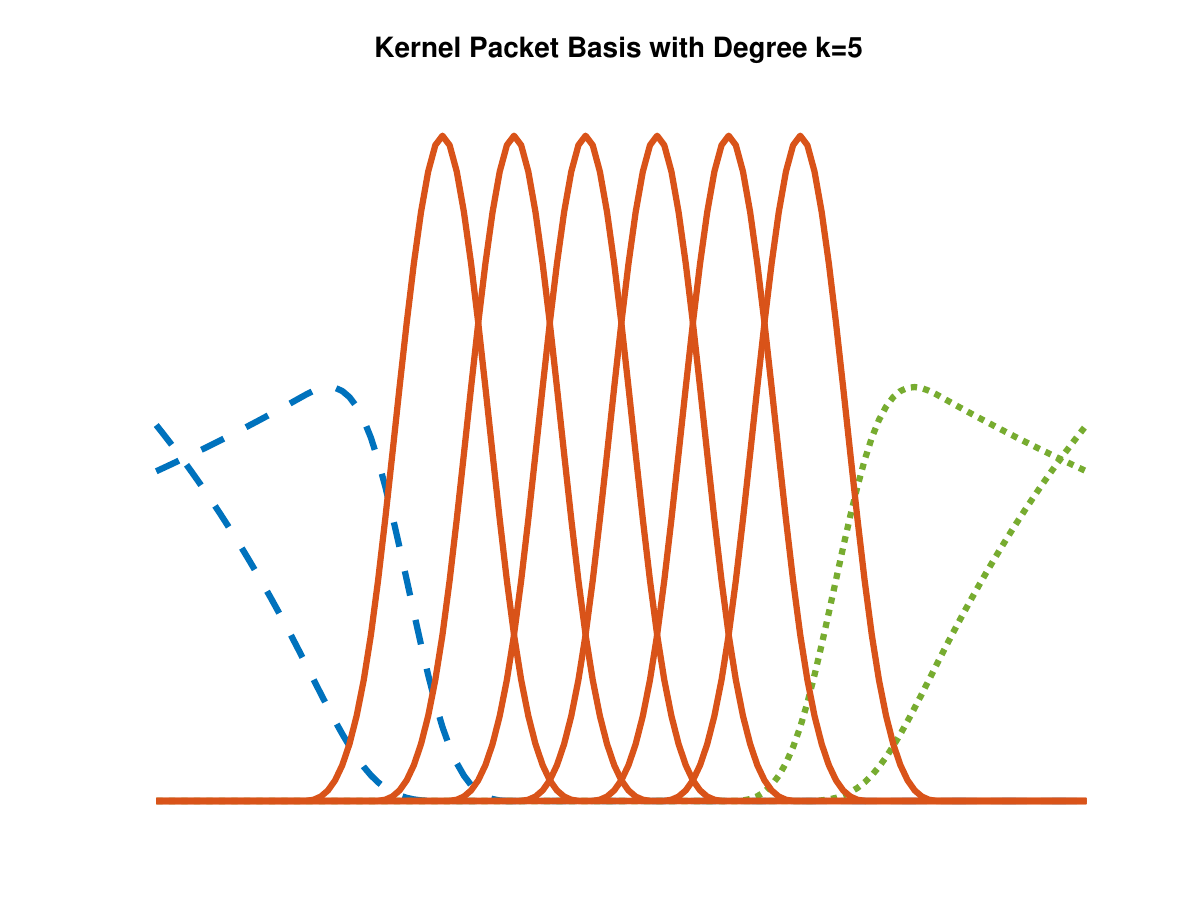}
\includegraphics[width=0.35\linewidth]{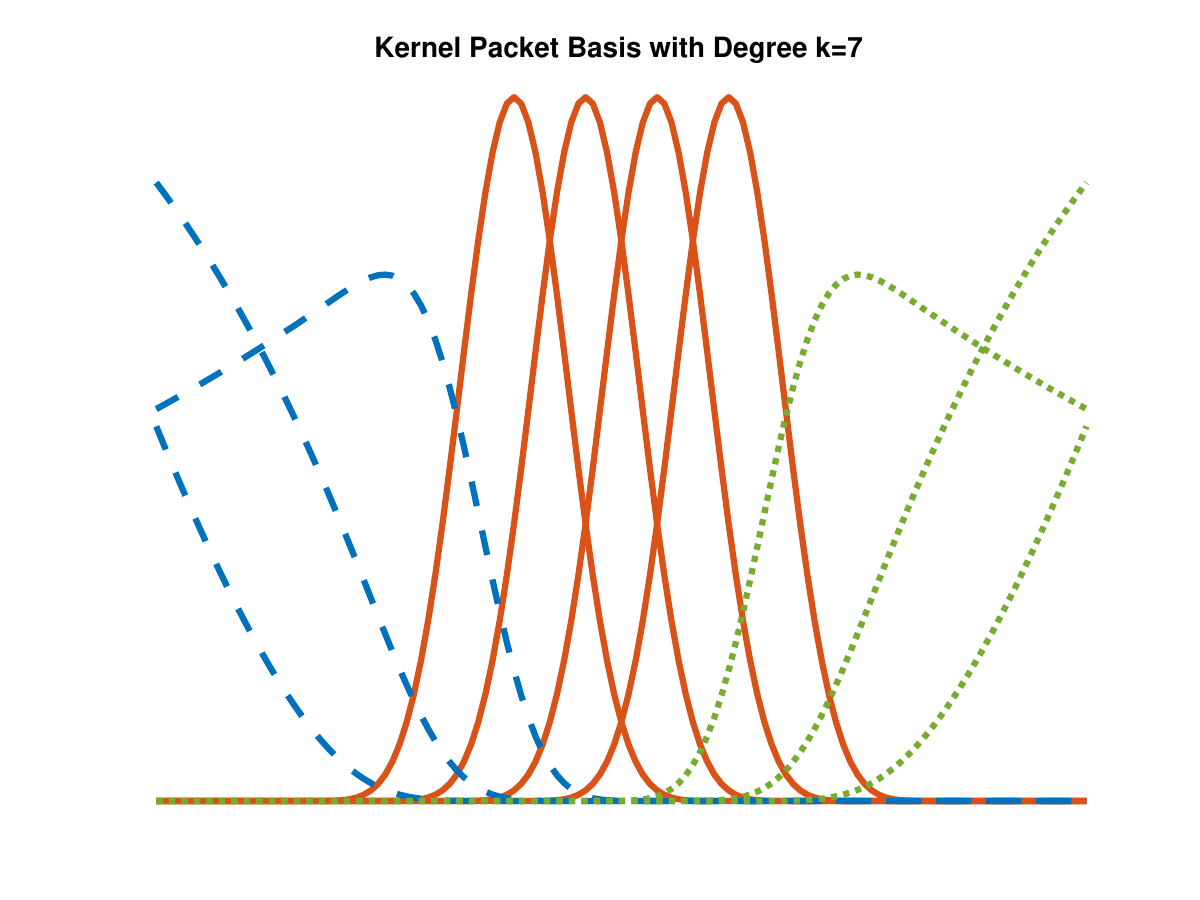}
\caption{KPs corresponding to  Mat\'ern-$3/2$ and $5/2$ correlations from \cite{chen2022kernel}.  KPs, left-sided KPs, and right-sided KPs are plotted in orange, blue, and green lines, respectively.} \label{fig:KP_basis}
\end{figure}



\cite{chen2022kernel}  showed that a sufficient condition for a KP system to form a KP basis is that each function in the system be \textit{irreducible}:

\begin{definition}
    A KP $\phi=\sum_{k=\underline{k}}^{\overline{k}} a_k K(\cdot,t_k)$ is called irreducible if 
    \begin{enumerate}
        \item No function with the form $\psi=\sum_{k=\underline{k}+1}^{\overline{k}} c_k K(\cdot,t_k)$ or $\psi=\sum_{k=\underline{k}}^{\overline{k}-1} c_k K(\cdot,t_k)$ with not-all-zero $c_k$'s can become a KP.
        \item There exist $t\in (t_{\underline{k}},t_{\underline{k}+1} ]$ and $t'\in [t_{\overline{k}-1},t_{\overline{k}})$ such that $\phi(t)\neq 0$ and $\phi(t')\neq 0$.
    \end{enumerate}
     Irreducible left- and right-sided KPs are defined analogously. A KP system is called minimal if each function in the system is irreducible.
\end{definition}

\begin{theorem}[\cite{chen2022kernel}]
\label{Th:KPbasis}
    A minimal KP system forms a KP basis. 
\end{theorem}

\section{Kernel Packets for State Space Models}
\label{sec:overview}
We first introduce the basic idea of constructing KPs from SS models, and then present an algorithm based on SS models for building a KP system.
\subsection{Kernel packets from SS models}
 
From a Gaussian-Markov process point of view, the existence of KPs is not entirely surprising. 
We  fix $t_1<\cdots<t_s$, and let $a_1,\ldots,a_s$ be undetermined coefficients. For any $t\geq t_s$, the Gauss-Markov property and linearity of the SS model \eqref{SS} imply that
\[\mathbb{E}[z(t)|z(t_1),\ldots,z(t_s)]=\mathbb{E}[z(t)|z(t_s)]=A(t,t_s)z(t_s),\quad\text{s.t.}\ A(t_s,t_s)=\mathbf{I}_m\]
for some deterministic function $A(t,t_s)$. Let covariance $R$ be defined as \eqref{eq:R_def}, then
\begin{eqnarray}
    \sum_{j=1}^s a_j R(t,t_j)=\mathbb{E}\left\{\mathbb{E}\left[z(t)\sum_{j=1}^s a_j z(t_j)^\top \Big|z(t_1),\ldots,z(t_s)\right]\right\} \nonumber\\
    =\mathbb{E}\left\{A(t,t_s)z(t_s)\sum_{j=1}^s a_j z(t_j)^\top \right\}=A(t,t_s)\underbrace{\sum_{j=1}^s a_j \mathbb{E}[z(t_s)z(t_j)^\top ]}_{(*)} .\label{Cov1}
\end{eqnarray}

Note that $(*)$ is independent of $t$, and (\ref{Cov1}) is zero if $(*)$ is zero. Then  if
\begin{eqnarray}\label{rightKPMarkov}
    \sum_{j=1}^s a_j R(t_s,t_j)=0,
\end{eqnarray}
we have $\sum_{j=1}^s a_j R(t,t_j)=0$ for all $t\geq t_s$. We call (\ref{rightKPMarkov}) the \textit{left-sided KP equations}. Analogously, we can have $\sum_{j=1}^s a_j R(t,t_j)=0$ for all $t\leq t_1$ if the \textit{left-sided KP equations} holds
\begin{eqnarray}\label{leftKPMarkov}
    \sum_{j=1}^s a_j R(t_s,t_j)=0,
\end{eqnarray}
which are derived from the ``backward Markov property'': $\mathbb{E}[z(t)|z(t_1),\ldots,z(t_s)]=\mathbb{E}[z(t)|z(t_1)]$ for $t\leq t_1$. We call the system of equations that simultaneously satisfies equations (\ref{rightKPMarkov}) and (\ref{leftKPMarkov}) the \textit{KP equations}. As shown in Theorem \ref{thm:existence}, a non-zero solution to the KP equations can lead to a KP, provided that $K$ is positive definite.

\begin{theorem}\label{thm:existence}
    Suppose $t_1<\cdots<t_s$, and $(a_1,\ldots,a_s)^\top $ is a non-zero vector satisfying both (\ref{rightKPMarkov}) and (\ref{leftKPMarkov}). Under Condition \ref{condition:full_rank}, the function $[\phi^{(j)}(\cdot)]_{j=1}^m=\sum_{i=1}^s a_i R(\cdot,t_i)$ satisfies $[\phi^{(j)}(t)]=0$ whenever $t\leq t_1$ or $t\geq t_s$ and $\phi^{(1)}$ is non-vanishing on $(t_1,t_s)$.
\end{theorem}
\begin{proof}
    $[\phi^{(j)}]$ are compactly supported on $(t_1,t_2)$ are obvious from \eqref{leftKPMarkov} and \eqref{rightKPMarkov}. It remains to prove that $\phi^{(1)}$ is non-vanishing. Note that for each $k=1,\ldots,s$,
    \begin{eqnarray*}
        \phi^{(1)}(t_k)=H\mathbb{E}\left\{z(t_k)\sum_{j=1}^s a_j z(t_j)^\top \right\}H^\top .
    \end{eqnarray*}
    Suppose $\phi^{(1)}(t_j)=0$ for $j=1,\ldots,s$. Then we have
    \begin{eqnarray}\label{VarContradiction}
        0=\sum_{k=1}^s a_k\phi^{(1)}(t_k)=H\mathbb{E}\left\{\sum_{k=1}^s a_kz(t_k)\sum_{j=1}^s a_j z(t_j)^\top \right\}H^\top =\operatorname{Var}\left(\sum_{k=1}^s a_k y(t_k)\right).
    \end{eqnarray}
    From Condition \ref{condition:full_rank}, $\CalL$ is invertible so kernel function $K$, which satisfies $\E[ \CalL[y](t)\CalL[y](\mu)]=\CalL_t\CalL_\mu K(t,\mu)=\delta_{t-\mu}$, is positive definite. However, because $(a_1,\ldots,a_s)$ is non-zero, (\ref{VarContradiction}) leads to contradiction.
\end{proof}
The question revolves around determining the minimal size of $s$ to allow for a non-zero solution to the KP equations. We have the following theorem:

\begin{theorem}
    \label{thm:fundamental-Covariance_centalKP} For any GP in the SDE form \eqref{eq:SODE} that satisfies Condition \ref{condition:full_rank}, the corresponding irreducible KP requires $s = 2m + 1$, and no smaller value of $s$ is possible.
 \end{theorem}
A detailed proof of Theorem \ref{thm:fundamental-Covariance_centalKP} is provided in Appendix \ref{sec:proof of ss}.

 \subsection{Algorithm and main theorem}
\begin{algorithm}[h]

\textbf{Input: }{ sorted point $t_1<\cdots<t_n$, covariance $R$ of a SS model of GP \eqref{eq:SODE}} 

\textbf{Return: }{banded matrix $\textbf{A}$ and kernel packets $\{\phi_{i}^{(j)}:i=1,\cdots,n;j=1,\cdots,m\}$} 

Define vector-valued function $R_1(t,\mu)=[R_{1,1}(t,\mu),\cdots,R_{1,m}(t,\mu)]$\;

\For{$i=1,2,\cdots, n$}{  

\If{$i\leq m$}{ 
$\text{\textbf{Left KPs: }solve for $a_l$ such that}:\  \sum_{l=i}^{i+m}a_lR_1(t_{i+m},t_l)=0$ ,\vspace{-1.5mm}
\begin{center}
$\text{let}\quad 
\mathbf{A}_{i,l}=a_l,\quad
\phi_{i}^{(j)} =
\sum_{l=i}^{i+m} a_l R_{1,j}(\cdot,t_l)$%
\hspace{1.5em}\refstepcounter{equation}(\theequation)\label{eq:lphi_def}
\end{center}\vspace{-1.5mm}
 }

\If{$m<i\leq n-m$}{ $\text{\textbf{Central KPs: }solve for $a_l$ such that}:\ \sum_{l=i-m}^{i+m}a_l[R_1(t_{i-m},t_l)\ R_{1}(t_{i+m},t_l)]=0$, \vspace{-1.5mm}

\begin{center}
    $\text{let}\quad \textbf{A}_{i,l}=a_l,\quad \phi_{i}^{(j)}=\sum_{l=i-m}^{i+m}a_l R_{1,j}(\cdot,t_l)$\hspace{1.5em}\refstepcounter{equation}(\theequation)\label{eq:cphi_def}\vspace{-1.5mm}
\end{center}
 }
\If{$i>n-m$}{ $\text{\textbf{Right KPs: }solve for $a_l$ such that}:\  \sum_{l=i-m}^{i}a_lR_1(t_{i-m},t_l)=0$,\vspace{-1.5mm}
\begin{center}
    $\text{let}\quad \textbf{A}_{i,l}=a_l,\quad \phi_i^{(j)}=\sum_{l=i-m}^{i}a_lR_{1,j}(\cdot,t_l)$ \hspace{1.5em}\refstepcounter{equation}(\theequation)\label{eq:rphi_def}\vspace{-1.5mm}
\end{center}
 }
} 
\caption{Computing transformation matrix $\textbf{A}$ and kernel packets $\phi^{(j)}_i$} \label{alg:KP}
\end{algorithm}
 
{{One of the main contributions of this work is the development of Algorithm~\ref{alg:KP}, a tractable method for computing the KP basis of a GP driven by an SDE as in~(\ref{eq:SODE}), where basis functions $\{\phi_i^{(j)}\}$ defined in~(\ref{eq:lphi_def})–(\ref{eq:rphi_def}) jointly form the KP system}}. In Algorithm \ref{alg:KP}, the covariance $R_{1,1}(t,\mu)=K(t,\mu)$ (see Eq.~\eqref{eq:R_def}) is clearly the kernel function for any SS formulation. For canonical SS formulation,  $R_{1,j}=D^{(j-1)}_tK(t,\mu)$ is simply the $(j-1)^{\text th}$ derivative of the kernel function with respect to time $t$.  For Mat\'ern kernels using SS formulation $z_2$ in Section \ref{sec:SS}, it can be show via direct calculations that Algorithm \ref{alg:KP} recovers the KP algorithm exactly as presented in \cite{chen2022kernel}. 

 Let $\CalK^{(j)}$ denote the function space $\text{span}\{R_{1,j}(\cdot,t_i)\}_{i=1}^n$ for $j=1,\cdots,m$.  It is essential for  any GP algorithm that the dimension of each  $\CalK^{(j)}$ is $n$. This ensures the invertibility of the covariance matrices $R_{1,j}(\mathbf{T},\mathbf{T})$.  Given Condition \ref{condition:full_rank}, we are prepared to present the main theorem of our paper, which states that KPs also form a basis of $\CalK^{(j)}$:
{{\begin{theorem}[Main Theorem]
\label{thm:main}
  Under Condition \ref{condition:full_rank}, $\{\phi_i^{(j)}\}_{i=1}^n$ forms a minimal KP system.
\end{theorem}}}

\begin{remark}
    Theorem \ref{thm:main} immediately gives the following three key properties of KPs:
\begin{enumerate}
    \item vector-valued function $\boldsymbol{\phi}^{(j)}(t)=\textbf{A}R_{1,j}(\mathbf{T},t)$ has  $\CalO(1)$ non-zero entries for any $t\in[t_0,T]$;
    \item matrix $\boldsymbol{\Phi}^{(j)}=\boldsymbol{\phi}^{(j)}(\mathbf{T})$ are banded matrices with band width $m-1$;
    \item $\boldsymbol{\Phi}^{(j)}$ are invertible because their columns are mutually linearly independent.
    \end{enumerate}
\end{remark}
A full proof of Theorem \ref{thm:main}, including all intermediate results and supporting lemmas, is presented in Appendix \ref{sec:proof of ss}.




\section{More Kernel Packets}
\label{sec:KPs_general_kernels}
{We begin by constructing KPs for one-dimensional combined kernels and then generalize them to multidimensional kernels in additive and product forms. }
\subsection{KPs for combined kernels}
\label{sec:combine_ker}

Kernel combinations, such as addition and multiplication, are powerful methods for creating  data-adaptive kernels.  We  show that KPs can be constructed for these two major types of combined kernels. In Theorems~\ref{thm:combine_ker_add} and~\ref{thm:combine_ker_multi}, we consider the covariance functions 
$R^{(1)}$ and $R^{(2)}$ corresponding to the SS model representations of two GPs $y_1$ and $y_2$ in the form of \eqref{eq:SODE}, both satisfying Condition~\ref{condition:full_rank}.

\begin{theorem}
\label{thm:combine_ker_add}
    Let $\psi=[\psi_1,\cdots,\psi_s]^\top $be any minimal spanning set of the function space \[\text{span}\{R_{1,j}^{(1)}(t_1,\cdot), \ R_{1,j}^{(1)}(t_{s+1},\cdot),\ R^{(2)}_{1,j}(t_1,\cdot), \ R^{(2)}_{1,j}(t_{s+1},\cdot) :\ j=1,\cdots,m\}.\]
    By solving $\sum_{j=1}^{s+1}a_j\psi(t_j)=0$, we have an irreducible KP of  $R^{(1)}+R^{(2)}$:
    \begin{equation}
        \label{eq:combine_KP_add}\sum_{j=1}^{s+1}a_j\left[R^{(1)}(t,t_j)+R^{(2)}(t,t_j)\right]=0,\quad \forall\ t\not\in(t_1,t_{s+1}).
    \end{equation}
\end{theorem}
\begin{figure}
    \centering

    \includegraphics[width=.24\linewidth]{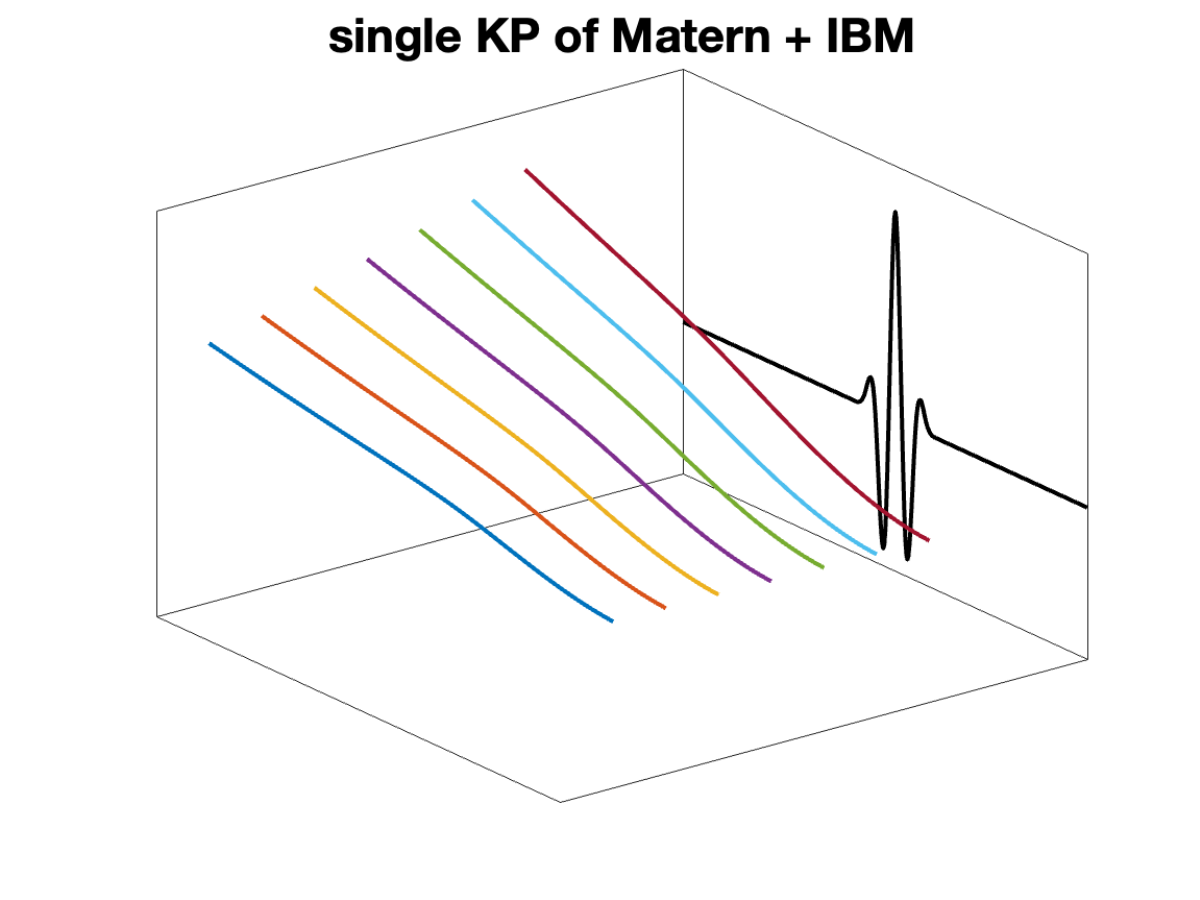}
     \includegraphics[width=.24\linewidth]{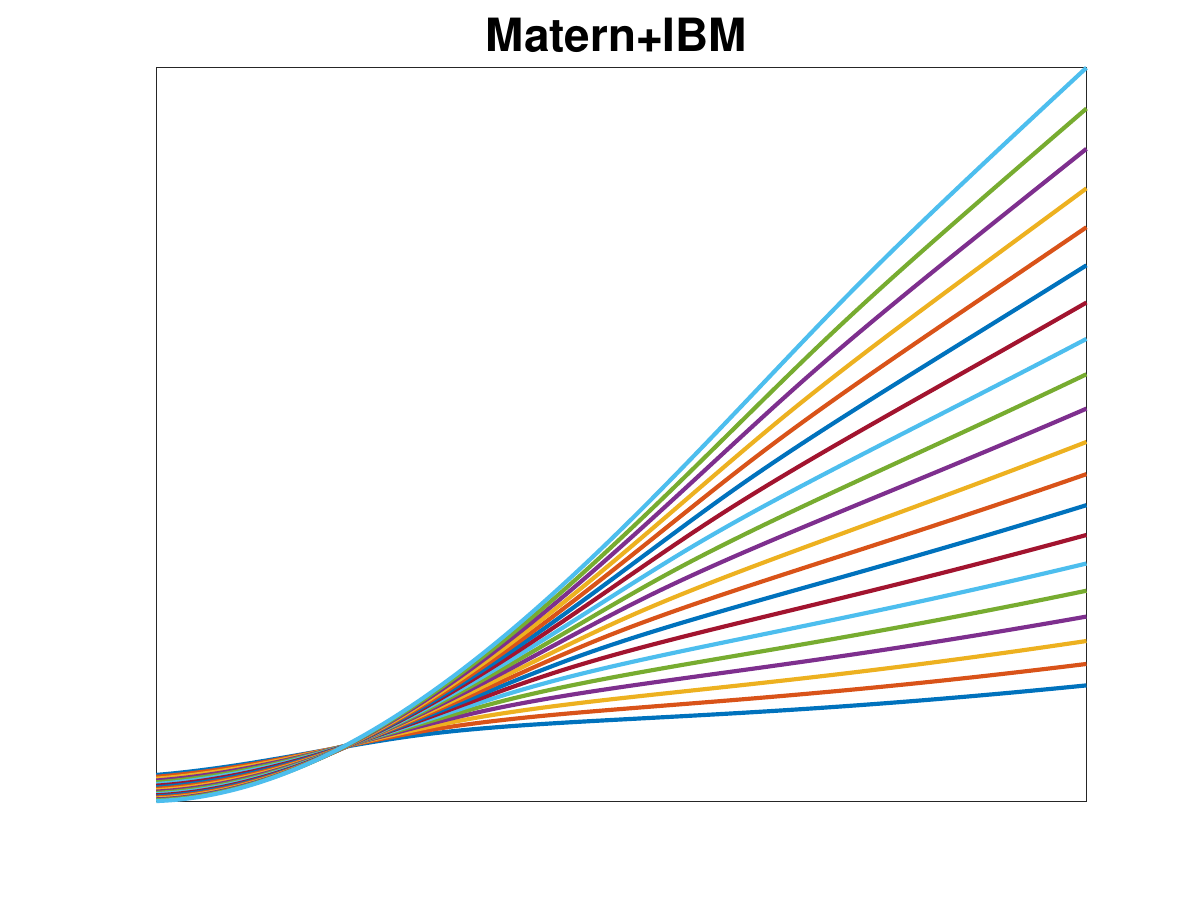}
    \includegraphics[width=.24\linewidth]{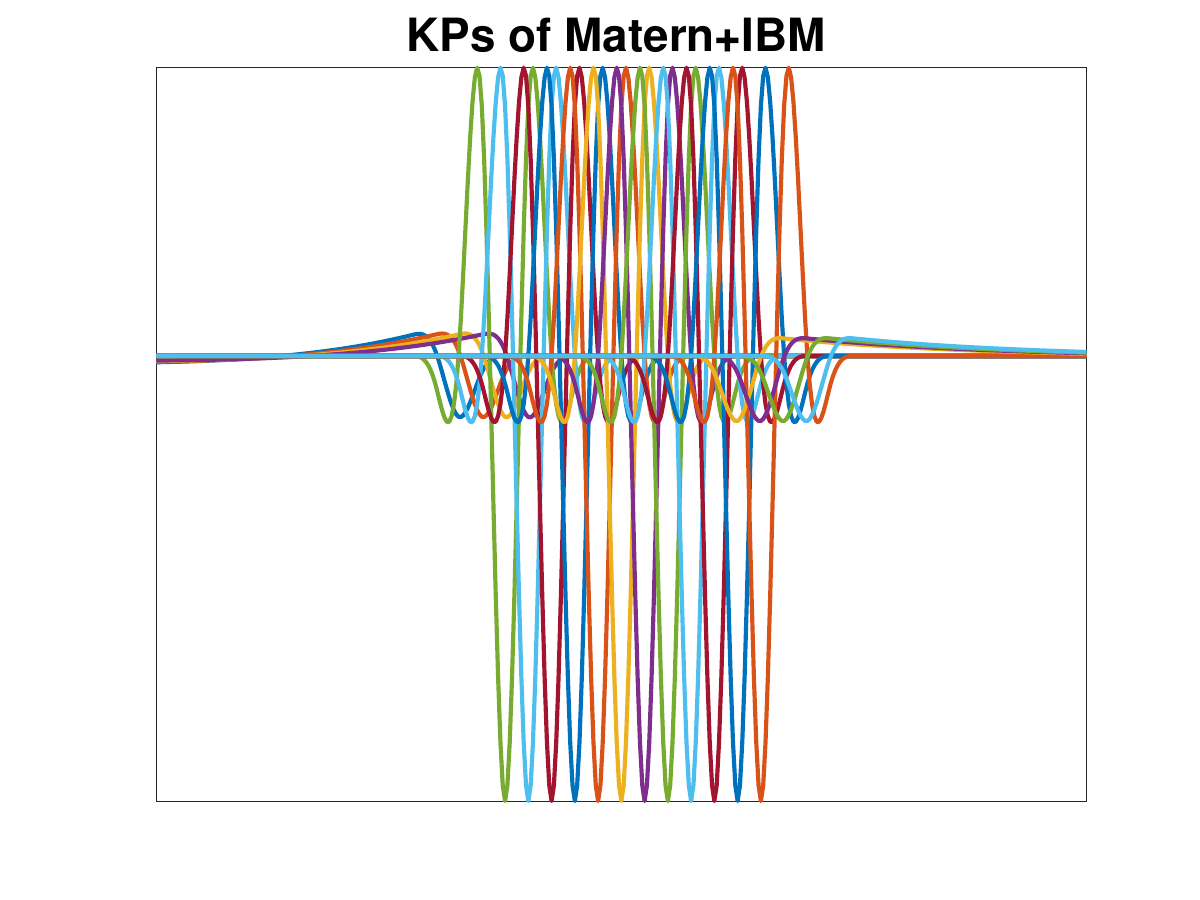}
    
    \includegraphics[width=.24\linewidth]{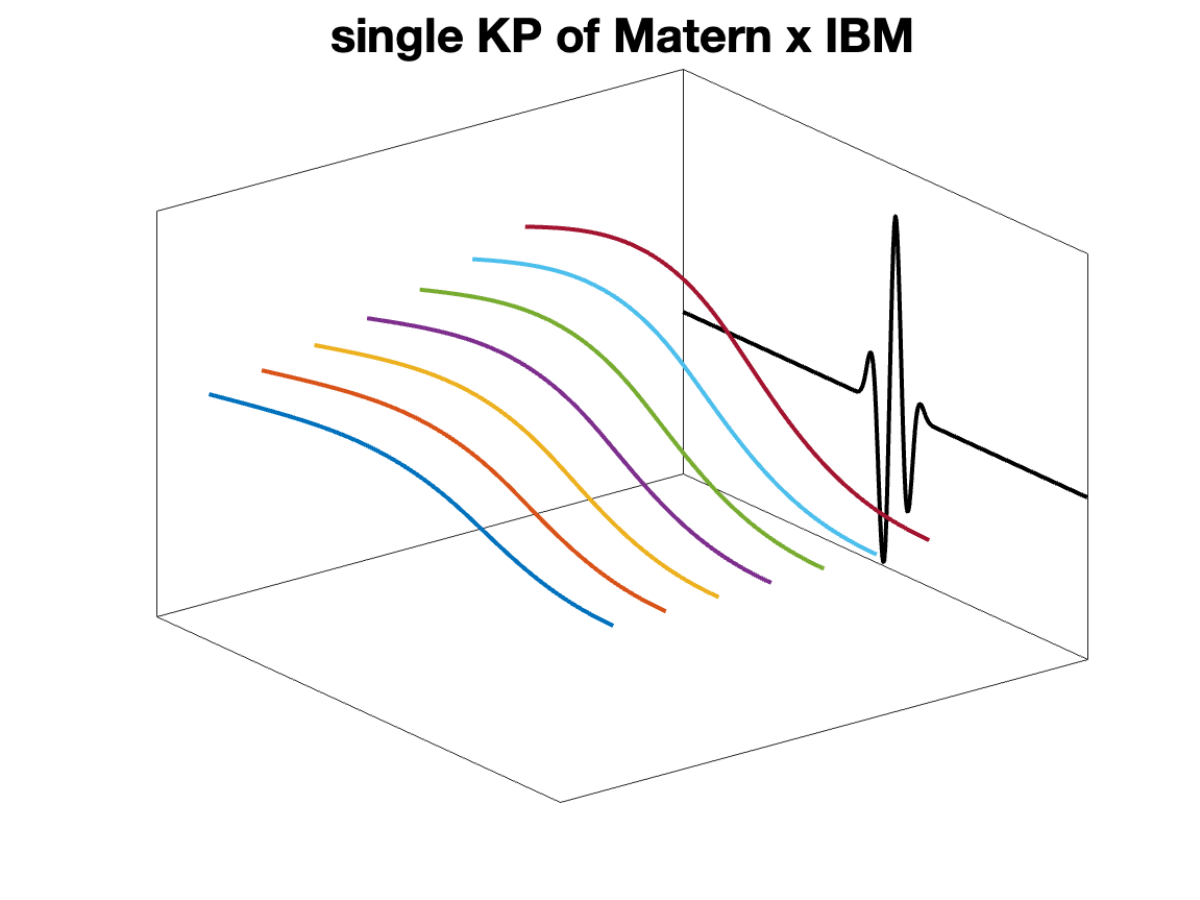}    
     \includegraphics[width=.24\linewidth]{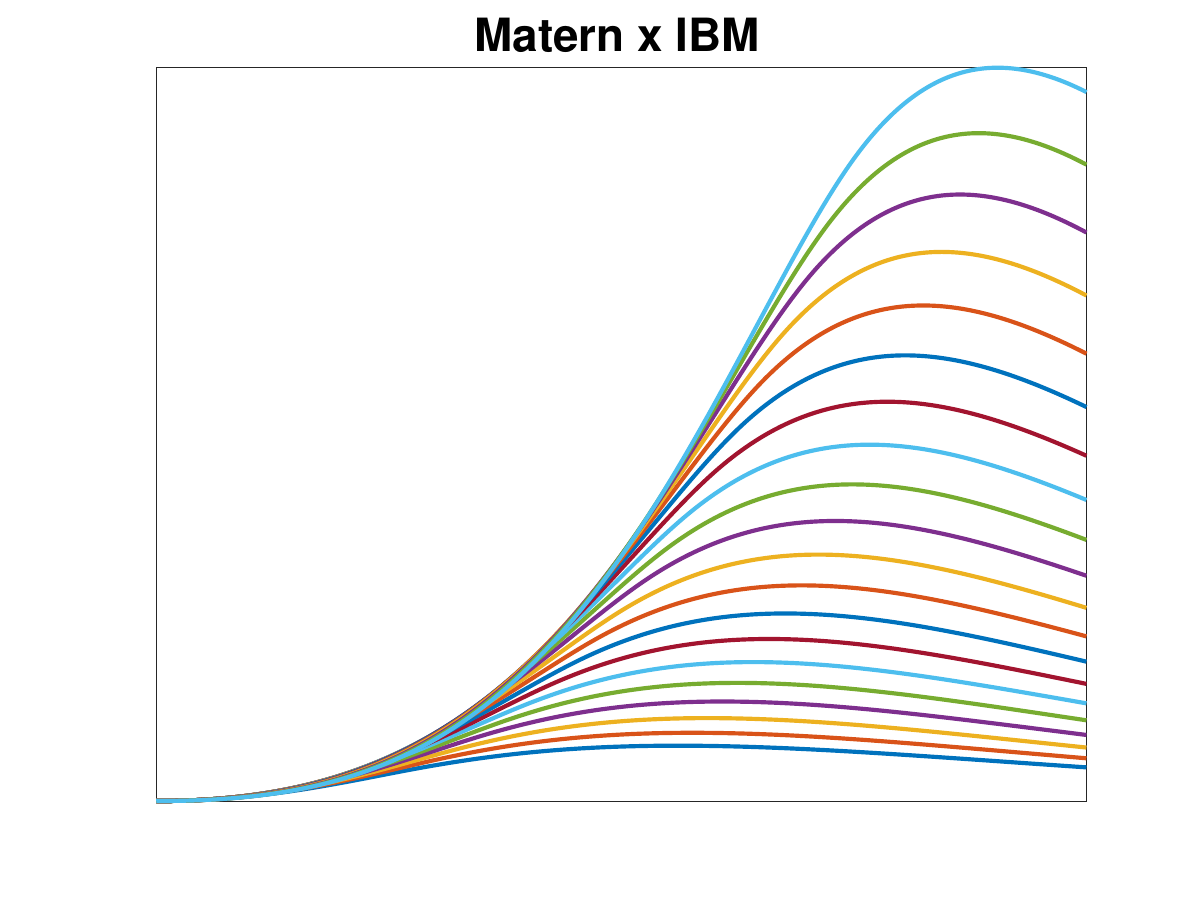}
     \includegraphics[width=.24\linewidth]{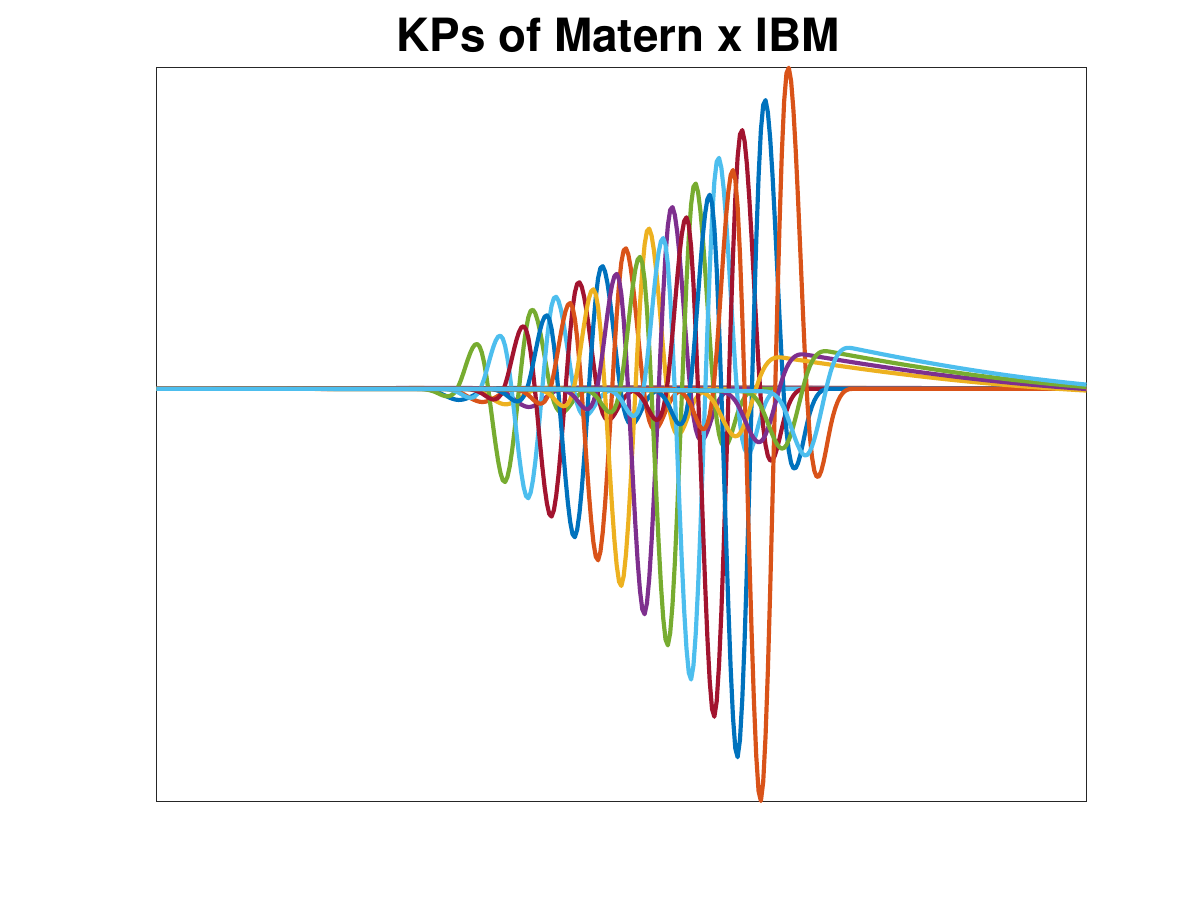}
    
    \caption{First column: KP is linear combinations of seven combined kernels; Middle column: Twenty kernel functions at $\{t_i=1+i/10\}_{i=1}^{20}$ that forms function spaces $\{K(\cdot,t_i)\}$; Last column: KP basis associated to combined kernel.}
    \label{fig:KP_experiment},
\end{figure}
\begin{theorem}
\label{thm:combine_ker_multi}
     Let   $\psi=[\psi_1,\cdots,\psi_{s}]^\top $ be any minimal spanning set of the function space 
    \begin{align*}
        \text{span}\{R^{(1)}_{1,j_1}(t_1,\cdot)R^{(2)}_{1,j_2}(t_1,\cdot),\ R^{(1)}_{1,j_1}(t_{s+1},\cdot)R^{(2)}_{1,j_2}(t_{s+1},\cdot):\  ,j_1,j_2=1,\cdots,m\}.
    \end{align*}
    Let $\mathbf{M}_1\bigotimes \mathbf{M}_2$ denote the Kronecker product of $\mathbf{M}_1$ and $\mathbf{M}_2$. By solving $\sum_{j=1}^{s+1}a_j\psi(t_j)=0$, we have an irreducible KP of  $R^{(1)}\bigotimes R^{(2)}$:
     \begin{equation}
        \label{eq:combine_KP_multiply}\sum_{j=1}^{s+1}a_j\left[R^{(1)}(t,t_j)\bigotimes R^{(2)}(t,t_j)\right]=0,\quad\forall\ t\not\in(t_1,t_{s+1}).
    \end{equation}
    
    
\end{theorem}

 Theorem~\ref{thm:combine_ker_add} gives the KP formulation for the sum of two SS models, while Theorem~\ref{thm:combine_ker_multi} corresponds to the Kronecker sum case. By the same reasoning, right- and left-sided KPs can be constructed for additive and multiplicative kernels. Since the right-sided KP is equivalent to the forward SS model, these constructions for right-sided KPs recover the SS model approach in \cite{solin2014explicit} for additive and multiplicative combined kernels.

We use the following two kernels to illustrate combined kernels: 
\begin{align}
    &K_{\text mat}(t,\tau)=\left(1+|t-\tau|\right)e^{-|t-\tau|}\label{eq:kernel_matern}\\
    &K_{\text ibm}(t,\ \tau)=\frac{t\tau(t\wedge \tau)}{2}-\frac{(t\wedge \tau)^3}{6},\quad t,\tau>0,\label{eq:kernel_bm}
\end{align}
where $K_{\text mat}$ is the Mat\'ern-$\frac{3}{2}$ kernel \citep{ Whittle54} and $K_{\text ibm}$ is the kernel of integrated Brownian motion (IBM). GPs induced by these two kernels are characterized by forward SDEs $\left(\partial_t+1\right)^2y(t)=W(t)$ and $\partial_{tt}y(t)=W(t)$, respectively. Therefore, they can also be represented by SS models with the following covariance functions:
\begin{equation*}
    R_{\text mat}(t,\tau)=
    \begin{cases}
        \exp\bigg\{ \begin{bmatrix}
        -1 & 1\\
        0& -1
    \end{bmatrix} t\bigg\}\exp\bigg\{ -\begin{bmatrix}
        -1 & -1\\
        0& -1
    \end{bmatrix} \tau\bigg\} \begin{bmatrix}
        1 & 1\\
        1& 2
    \end{bmatrix} & \text{if}\ t\geq \tau\\
     & \\
    \begin{bmatrix}
        1 & 1\\
        1& 2
    \end{bmatrix} 
        \exp\bigg\{ \begin{bmatrix}
        -1 & 0\\
        1& -1
    \end{bmatrix} \tau \bigg\}\exp\bigg\{- \begin{bmatrix}
        -1 & 0\\
        1& -1
    \end{bmatrix} t\bigg\}  &  \text{if}\ t\leq \tau
    \end{cases}
\end{equation*}

\begin{equation*}
    R_{\text ibm}(t,\tau)=
    \begin{cases}
        \exp\bigg\{ \begin{bmatrix}
        0 & 1\\
        0& 0
    \end{bmatrix} t\bigg\}\exp\bigg\{ -\begin{bmatrix}
        0 & 1\\
        0& 0
    \end{bmatrix} \tau\bigg\} \begin{bmatrix}
        \frac{\tau^3}{3} & \frac{\tau^2}{2}\\
        \frac{\tau^2}{2}& \tau
    \end{bmatrix} & \text{if}\ t\geq \tau\\
     & \\
    \begin{bmatrix}
        \frac{t^3}{3} & \frac{t^2}{2}\\
        \frac{t^2}{2}& t
    \end{bmatrix} 
        \exp\bigg\{ \begin{bmatrix}
        0 & 0\\
        1& 0
    \end{bmatrix} \tau \bigg\}\exp\bigg\{- \begin{bmatrix}
        0 & 0\\
        0& 1
    \end{bmatrix} t\bigg\}  &  \text{if}\ t\leq \tau
    \end{cases}
\end{equation*}

By applying Theorems \ref{thm:combine_ker_add} and \ref{thm:combine_ker_multi} to covariances  $R_{\text mat}$ and $R_{\text ibm}$, direct calculations show that for any seven consecutive points $\{t_i\}_{i=1}^7$, there exist seven coefficients $\{a_i\}_{i=1}^7$ such that the linear combinations $\sum_{i=1}^7a_i[R_{\text mat}(\cdot,t_i)+R_{\text ibm}(\cdot,t_i)]$ and $\sum_{i=1}^7a_i[R_{\text mat}(\cdot,t_i)R_{\text ibm}(\cdot,t_i)]$ are compactly supported on $[t_1,t_7]$. As shown in Figure \ref{fig:KP_experiment},  we computed the KP basis for 20 consecutive points $\{t_i=1+i/10\}_{i=1}^{20}$ and run Algorithm \ref{alg:KP} to convert  $\{K(\cdot,t_i)\}_{i=1}^{20}$ to KPs $\{\phi_i\}_{i=1}^{20}$ for additive kernel $K=K_{\text mat}+K_{\text ibm}$,  and product kernel $K=K_{\text mat}K_{\text ibm}$, which are all compactly supported.

 \subsection{Multidimensional Kernel Packets}
Multi-dimensional KPs for grid-based observations—such as full, sparse, and composite grids \cite{plumlee2021composite}—can be constructed via tensor products of KPs \citep{chen2022kernel}, but they are not applicable to scattered data. A KP-based backfitting algorithm for GP regression with additive kernels and scattered data was recently proposed in \cite{Zou17092025}, though it does not extend to product kernels.

In this subsection, we illustrate how to generalize KPs for additive and product kernels for multi-dimensional scattered data. The main idea remains similar—a linear combination of a finite number of kernel functions yields a function with compact support properties, as shown in Figure \ref{fig:KP_2D_matern}.

 \subsubsection{Additive Kernels}
 \label{sec:additive_kernel_mult}
Based on Theorem \ref{thm:combine_ker_add}, it is evident that the theorem remain valid even when the combined kernel is from adding kernels across varying dimensions. This implies that KPs exist for additive kernels at multi-dimensional input points. Here, we consider GPs $y^{(d)}$ all having the form \eqref{eq:SODE} and satisfying Condition \ref{condition:full_rank} in dimension $d$. Suppose each of $y^{(d)}$ has a specific SS model representation with  covariance $R^{(d)}$ with kernel $K_d=R^{(d)}_{1,1}$. 

\begin{theorem}   
\label{thm:combine_ker_add_multi_dim}  
    Given any $s=2mD+1$ scattered points $\{\bold{t}_i=(t_{i,1},t_{i,2},\cdots,t_{i,D})\}_{i=1}^{s}$,  define, for each dimension, the minimum and maximum points  $\underline{t}^{(d)}=\min_{i}\{t_{i,d}\}$ and $\overline{t}^{(d)}=\max_i\{t_{i,d}\}$.  Define   $R=\sum_{d=1}^D R^{(d)}$.   Define vector-valued function $H$ as
    \[H(\Bt)=[\phi_{j,d}(t_d),\psi_{j,d}(t_d)]_{j\in\{1,\cdots,m\},d\in\{1,\cdots,D\}},\ \text{where}\ \phi_{j,d}=R^{(d)}_{1,j}(\underline{t}^{(d)},\cdot), \ \psi_{j,d}=R^{(d)}_{1,j}(\overline{t}^{(d)},\cdot). \]
     By solving the  KP equations $\sum_{i=1}^{s}a_iH(\bold{t}_i)=0$, we have the irreducible KP:
    \begin{equation}
    \label{eq:combine_KP_add_multi_dim}\sum_{i=1}^{s+1}a_iR(\bold{t},\bold{t}_i)=0,\quad \forall\ \Bt\in U=\times_d\{(-\infty,\underline{t}^{(d)})\bigcup(\overline{t}^{(d)},\infty)\}.
    \end{equation}
\end{theorem}

 \begin{figure}
    \centering
    \includegraphics[width=.35\linewidth]{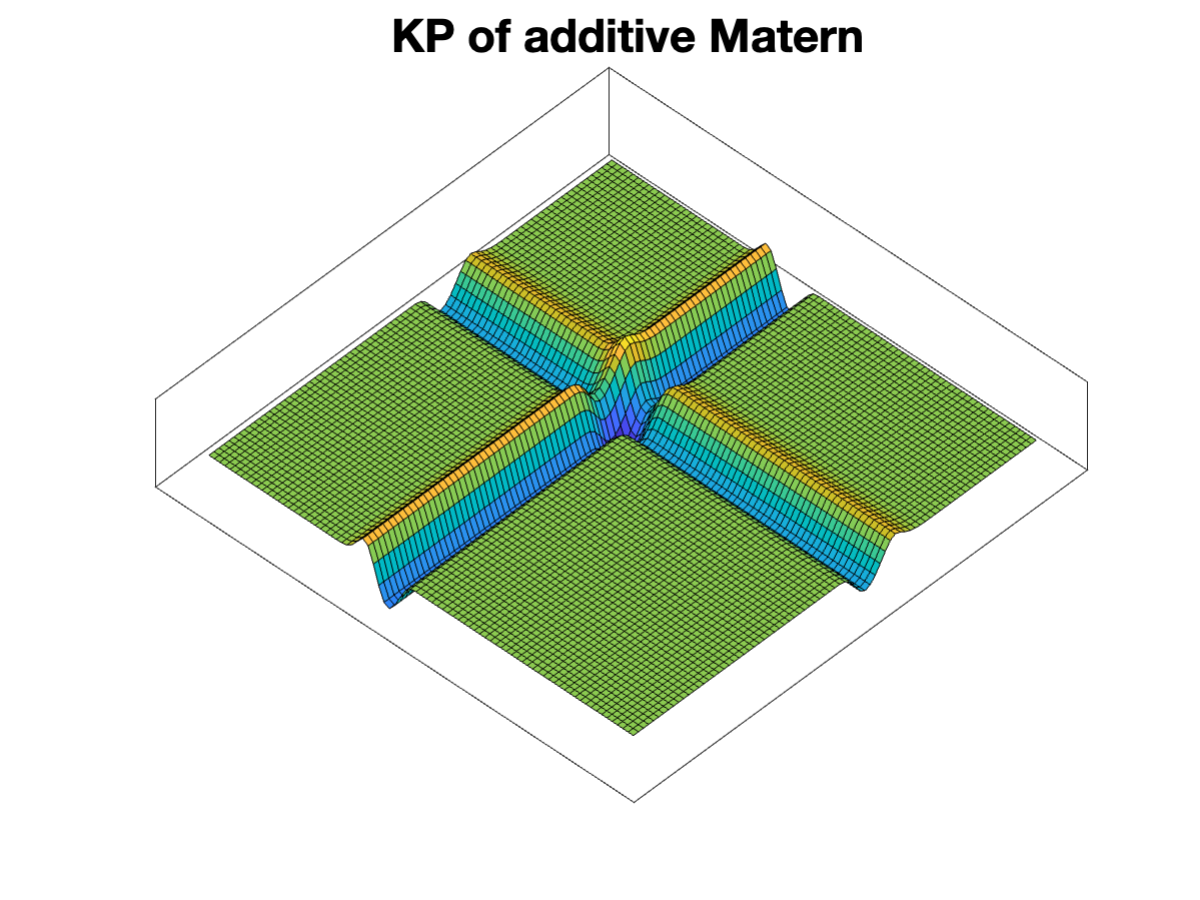}
     \includegraphics[width=.35\linewidth]{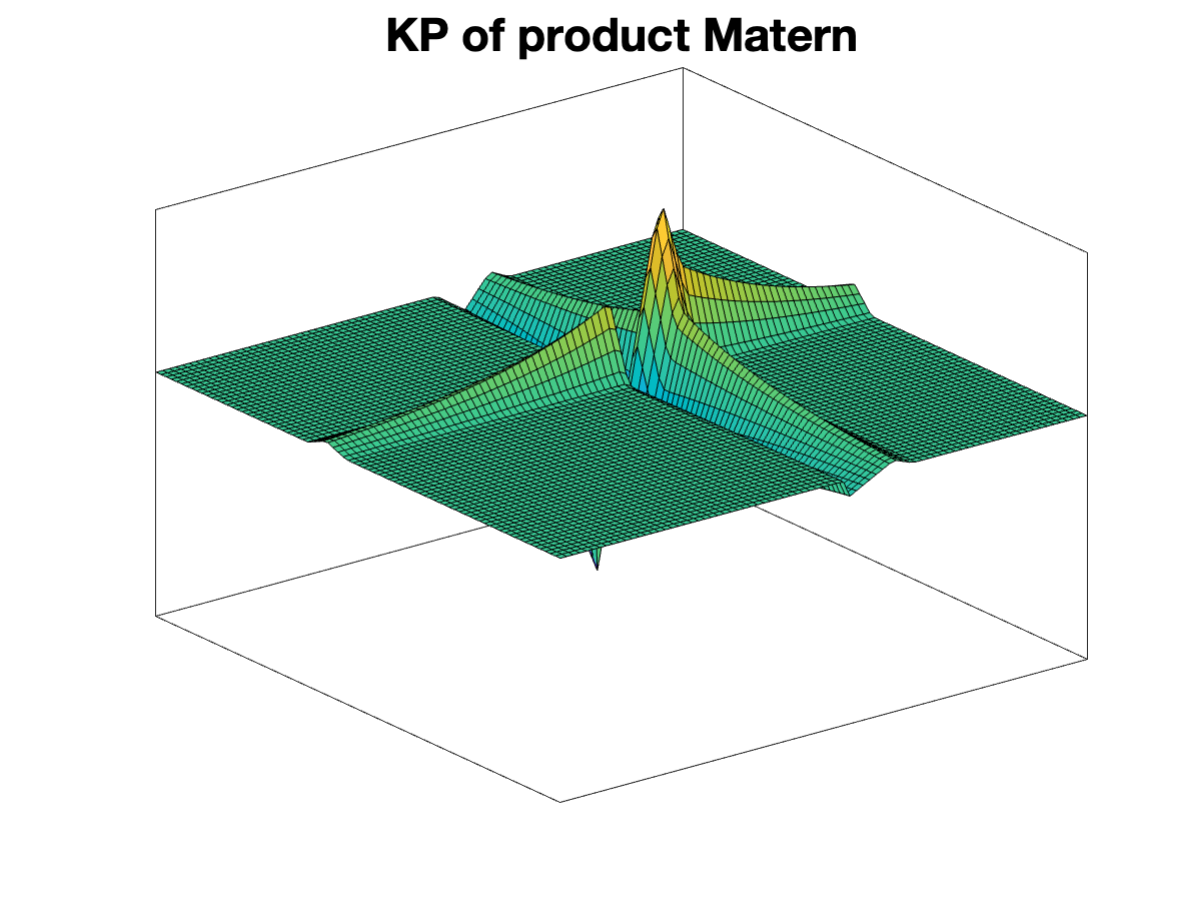}
    
    \caption{ Two-dimensional KP of additive (left) and product (right) Mat\'ern kernels}
    \label{fig:KP_2D_matern}
\end{figure}
Theorem \ref{thm:combine_ker_add_multi_dim} extends the idea of Theorem \ref{thm:combine_ker_add} to a multi-dimensional setting. While Theorem \ref{thm:combine_ker_add} focuses on combining two sets of fundamental solutions related to kernels within the same dimension through direct sum, Theorem \ref{thm:combine_ker_add_multi_dim} follows this approach by facilitating the direct sum of fundamental solutions from kernels in different dimensions.

We use a two-dimensional GPs with  additive Mat\'ern kernel to illustrate Theorem \ref{thm:combine_ker_add_multi_dim}:
\[K(\bold{t},\bold{t}')=(1+|t_1-t_1'|)\exp(-|t_1-t_1'|)+(1+|t_2-t_2'|)\exp(-|t_2-t_2'|).\]
By canceling common factors, the associated vector-valued function $H$ can be reduced to:
\[H(t_1,t_2)=[e^{-t_1},\ t_1e^{-t_1},\ e^{t_1},\ t_1e^{t_1},\ e^{-t_2},\ t_2e^{-t_2},\ e^{t_2},\ t_2e^{t_2}]^\top .\]
In this numerical example, we draw the following nine points uniformly from $[0,1]^2$.  By solving the KP equations given in Theorem \ref{thm:combine_ker_add_multi_dim} , we can have  a two-dimensional KP as shown in the left plot of Figure \ref{fig:KP_2D_matern}.


    

\subsubsection{Product Kernels}
\label{sec:product_kernel}
Based on Theorem \ref{thm:combine_ker_multi}, it is also evident that the theorem remain valid even when the combined kernel is from product kernels across varying dimensions. This implies that KPs exist for product kernels at multi-dimensional input points, i.e., kernel of the form $K(\Bt,\Bt)=\prod_{d=1}^DK_d(t_d,t_d')$. Similar to the additive setting in Section \ref{sec:additive_kernel_mult}, we consider covariances $R^{(d)}$  of GPs $y^{(d)}$ all having the form \eqref{eq:SODE} and meeting Condition \ref{condition:full_rank} in dimension $d$.

Unlike KP systems for multi-dimensional additive kernels that utilize a direct sum approach for each dimension, KP systems for multi-dimensional product kernels employ a tensor product of KP systems across dimensions. We can follow the basic idea  in Theorem \ref{thm:combine_ker_multi} for product form that involves kernels in different dimensions.
\begin{theorem}
    \label{thm:combine_ker_prod_multi_dim}
  Given any $s=(2m)^D+1$ scattered points $\{\Bt_i=(t_{i,1},t_{i,2},\cdots,t_{i,D})\}_{i=1}^s$, define, for each dimension, the minimum and maximum points  as $\underline{t}^{(d)}=\min_{i}\{t_{i,d}\}$ and $\overline{t}^{(d)}=\max_i\{t_{i,d}\}$. Define  $R=\bigotimes_{d=1}^D R^{(d)}$. 
 Define vector-valued function $H$ as 
 \[H(\Bt)=\bigotimes_{d=1}^D \left[R^{(d)}_{1}(\underline{t}^{(d)},\cdot),R^{(d)}_{1}(\overline{t}^{(d)},\cdot) \right].\]
 By solving the KP equations $\sum_{i=1}^sa_i H(\Bt_i) =0$, we have the irreducible KP :
    \begin{equation}
    \label{eq:combine_KP_prod_multi_dim}\sum_{i=1}^{s+1}a_i R(\bold{t},\bold{t}_i)=0,\quad \forall\ \Bt\in U=\times_d\{(-\infty,\underline{t}^{(d)})\bigcup(\overline{t}^{(d)},\infty)\}.
    \end{equation}
\end{theorem}
We  use the following two-dimensional product Mat\'ern kernel to illustrate Theorem \ref{thm:combine_ker_prod_multi_dim}:
\[K(\bold{t},\bold{t}')=\exp(-|t_1-t_1'|)\exp(-|t_2-t_2'|).\]
By canceling common factors, the covariances $R_1$ of Mat\'ern kernel $\exp(-|t_d-t_d'|)$ is equivalent to  $[e^{-t_d}\  e^{t_d}\ ]$. So  the function $H$ in Theorem \ref{thm:combine_ker_prod_multi_dim} is:
\begin{align*}
    H(t_1,t_2)=&[e^{-t_1}\  e^{t_1}\ ]^\top \bigotimes [e^{-t_2}\  e^{t_2}\ ]^\top=[e^{-t_1-t_2},\  e^{t_1-t_2},\ e^{-t_1+t_2},\ e^{t_1+t_2}\ ]^\top 
\end{align*}
In this numerical example, we draw the following five points uniformly from $[0,1]^2$.
 By solving KP equations given in Theorem \ref{thm:combine_ker_prod_multi_dim} , we can have  a two-dimensional KP as shown in the right plot of Figure \ref{fig:KP_2D_matern}.

\section{Training and Prediction Algorithms of GPs via KPs}
For GP $y\sim\CalN(0,K_{\Btheta})$ where $K_{\Btheta}$ is a  kernel parametrized by $\Btheta$.  Suppose we observe $n$ noisy data $(\BT,\BZ)=\{(\Bt_i,Z(\Bt_i))\}_{i=1}^n$ , where each data is  $Z(\Bt_i)=y(\Bt_i)+\varepsilon$ with $\varepsilon\sim \mathcal{N}(0,\sigma^2_y)$. In this case, the covariance of the
observed noisy responses is $\text{Cov}\big(Z(\Bt_i),Z(\Bt_j)\big)= K(\Bt_i,\Bt_j)+\sigma_y^2\mathbb{I}(i=j)$. In other words, the covariance matrix $\text{Cov}(Z,Z)$ is $K(\BT,\BT)+\sigma_y^2\bold{I}_n$. The posterior predictor at a new point $\Bt^*$ is also normal distributed with the following conditional mean and variance:
\begin{align}
          &\E\left[y(\Bt^*)\big| \BZ\right]= K(\Bt^*,\BT)\left[K(\BT,\BT)+{\sigma_y^2}\bold{I}\right]^{-1}\BZ,\label{eq:conditional-mean-noisy}\\
          & \text{Var}\left[y(\Bt^*)\big|\BZ\right]=K(\Bt^*,\Bt^*)-K(\Bt^*,\BT)\left[K(\BT,\BT)+{\sigma_y^2}\bold{I}\right]^{-1}K(\BT,\Bt^*),\label{eq:conditional-variance-noisy}
\end{align}
and the log-likelihood function of $\Btheta$ given data $\BZ$ is:
\begin{equation}\label{eq:loglike-noisy}
    L(\Btheta)=-\frac{1}{2}\left[ \log \det( K_{\Btheta}(\BT,\BT)+\sigma^2_y\bold{I}) + \BZ^\top \big[ K_{\Btheta}(\BT,\BT)+\sigma_y^2\bold{I}\big]^{-1}\BZ\right].
\end{equation}
In this section, we first propose algorithms for efficient computation of \eqref{eq:conditional-mean-noisy}–\eqref{eq:conditional-variance-noisy} for one-dimensional GPs without employing any approximation, and then extend them to the multi-dimensional case.

\label{sec:prediction_algorithm}

\subsection{One-dimensional Gaussian Processes}
Suppose input $\BT=\{\Bt_i=t_i\}_{i=1}^n$ is one dimensional and $y$  follows a parametrized SDE as follows
\begin{equation}
\label{eq:SODE_parametrize}
    y^{(m)}(t)+c_{m-1}(t;\Btheta)y^{(m-1)}.(t)+\cdots+c_0(t;\Btheta)y(t)=W(t)
\end{equation}
Then \eqref{eq:conditional-mean-noisy}, \eqref{eq:conditional-variance-noisy}, and \eqref{eq:loglike-noisy} can be calculated in $\CalO(m^3 n)$ because, from Algorithm \ref{alg:KP},  $K_{\Btheta}(\cdot,\BT)\BA_{\Btheta}=\BPhi^\top _{\Btheta}(\cdot)$ and $\BA_{\Btheta}K_{\Btheta}(\BT,\cdot)=\BPhi_{\Btheta}(\cdot)$ where $\BA_{\Theta}$ is the transformation matrix,  $K_\theta=R_{1,1}$, and $[\BPhi_{\Btheta}(\cdot)]_i=\phi^{(1)}_i$ KP basis functions obtained by Algorithm \ref{alg:KP} with input $K_{\Btheta}$ and sorted points $\BT$. So the covariance matrix $K_{\Btheta}(\BT,\BT)+\sigma_y^2\bold{I}$ admits the following factorization
\begin{equation}
     \label{eq:Kmat2KPmat-noisy}
     K_{\Btheta}(\BT,\BT)+\sigma_y^2\bold{I}=\BA_{\Btheta}^{-1}\big(\BPhi_{\Btheta}(\BT)+\sigma_y^2\BA_{\Btheta}\big)=\big(\BPhi_{\Btheta}(\BT)^\top +\sigma_y^2\BA_{\Btheta}\big)\BA_{\Btheta}^{-1},
\end{equation}
The computational time complexity of Algorithm \ref{alg:KP} is $\CalO(m^3n)$ obviously for it solves an $m\times (m+1)$ system in each of its $n$ total iterations.

By substituting \eqref{eq:Kmat2KPmat-noisy}  into \eqref{eq:conditional-mean-noisy}, \eqref{eq:conditional-variance-noisy}, and \eqref{eq:loglike-noisy}, we can obtain:
\begin{align}
          &\E\left[y(t^*)\big|\BZ\right]= \BPhi^\top (t^*)\left[\BPhi(\BT)+{\sigma_y^2}\BA\right]^{-1}\BZ,\label{eq:conditional-mean-KP-noisy}\\
          & \text{Var}\left[y(t^*)\big|\BZ\right]=K(t^*,t^*)-\BPhi^\top (t^*)\left[\BPhi(\BT)+{\sigma_y^2}\BA\right]^{-1} \BA^{-T}\BPhi(t^*)\label{eq:conditional-variance-KP-noisy}
\end{align}
and 
\begin{align}
    L(\Btheta)=&-\frac{1}{2}\bigg[ \log \det\big(\BPhi_{\Btheta}(\BT)+\sigma_y^2\BA_{\Btheta}\big)-\log \det(\BA_{\Btheta})+\BZ^\top \BA_{\Btheta}\big[\BPhi_{\Btheta}(\BT)+\sigma_y^2\BA_{\Btheta}\big]^{-1}\BZ\bigg].\label{eq:loglike-noisy-KP} 
\end{align}
According to Main Theorem \ref{thm:main} $\BPhi_{\Btheta}(\BT)$ and $\BA_{\Btheta}$ are banded matrices with bandwidth $m-1$ and $m$, respectively. Therefore, the matrix $ \BPhi_{\Btheta}(\BT)+\sigma_y^2\BA_{\Btheta}$ is also a banded matrix with bandwidth $m$. Time complexity for computing this sum is $\CalO( mn)$.  Now we present algorithms for the efficient computation of \eqref{eq:conditional-mean-KP-noisy}, \eqref{eq:conditional-variance-KP-noisy}, and \eqref{eq:loglike-noisy-KP}, as each requires distinct computational approaches.

\subsubsection{Computations of posterior mean \eqref{eq:conditional-mean-KP-noisy}}
\label{subsection:posterio_mean}
The training of \eqref{eq:conditional-mean-KP-noisy} can be regarded as solving the vector $C=\left[\BPhi(\BT)+{\sigma_y^2}\BA\right]^{-1}\BZ$. Because, as we mentioned previously, $ \BPhi(\BT)+\sigma_y^2\BA$ is a banded matrix with band width $m$, the computation of $C$ can be done in $\CalO(m^3n)$ time by applying  banded matrix solver. For example, the algorithm based on the LU decomposition in \cite{davis2006direct} can be applied to solve the equation. MATLAB also provides convenient and efficient builtin functions, such as \texttt{mldivide} or \texttt{decomposition}, to solve sparse banded linear system in this form.  

The posterior mean at a new point $t^*$ is calculated through the computation of the inner product $\BPhi^\top (t^*)C$, with $C$ being determined during the training phase. From Main Theorem \ref{thm:main},    number of  non-zero entries of $\BPhi(t^*)$ is  at most $m$. So the time complexity for computing the inner product is then $\CalO(\log n)$ for searching indices of the non-zero entries, or even $\CalO(1)$ if the smallest $i$ such that $t_i>t$ is known.

\subsubsection{Computations of posterior variance \eqref{eq:conditional-variance-KP-noisy} }
\label{subsec:posterior_variance}

In the computation of the conditional variance as specified by \eqref{eq:conditional-variance-KP-noisy}, the sparse structure of $\BPhi(t^*)$, which contains at most $m$ non-zero consecutive entries for any given point $t^*$, significantly reduces the computational complexity. For the training process, it is sufficient to calculate the $m$-band of the matrix $\left[\BPhi(\BT)+{\sigma_y^2}\BA\right]^{-1} \BA^{-T}$, since only these parts of the matrix are required for computing $\BPhi(t^*)^\top \left[\BPhi(\BT)+{\sigma_y^2}\BA\right]^{-1} \BA^{-T}\BPhi(t^*)$. This reduction of computation ensures efficiency, as entries outside the $m$-band do not contribute to the calculation of the conditional variance for any $t^*$.
\begin{algorithm}
\caption{Computing the $m$-band of $\BPsi^{-1} \BA^{-T}$}\label{alg:band_PhiA}
\textbf{Input}: banded matrices $\BPsi$ and $\BA$ \vspace{+1.5mm}

 \textbf{Output}:  $[\BPsi^{-1}\BA^{-T}]_{i,j}$ for $|i-j|\leq m$\vspace{+1.5mm}
 
Define matrix blocks $\BFH_i^-,\BFH_i,\BFH_i^+$ of $[h_{i,j}]:=\BA^\top \BPsi$ as 
\begin{align}
    \BFH_i^-=&\begin{bmatrix}
 h_{s_i,s_i-2m}  &\cdots& h_{s_i,s_i-1}\\
   &\ddots &  \vdots \\
    & &h_{s_{i+1}-1,s_{i}-1}
\end{bmatrix},\ \ \BFH_i=\begin{bmatrix}
 h_{s_i,s_i}  &\cdots& h_{s_i,s_{i+1}-1}\\
  \vdots &\ddots &  \vdots \\
  h_{s_{i+1}-1,s_{i}} &\cdots &h_{s_{i+1}-1,s_{i+1}-1}
\end{bmatrix},\nonumber\\
\BFH_i^+=&\begin{bmatrix}
 h_{s_i,s_{i+1}}  & & \\
  \vdots &\ddots &   \\
  h_{s_{i+1}-1,s_{i+1}} &\cdots &h_{s_{i+1}-1,s_{i+2}-1}
\end{bmatrix}\label{eq:tridiagonal_block}
\end{align}
where $i=1,\cdots,I$, $I=\lceil \frac{n}{2m}\rceil$, $s_i=(i-1)2m+1$, and $s_{I+1}-1=\min\{n,2m I \}$ \\
{\small(Note: {{$\BA^\top \BPsi
$ is a $2m$-banded matrix, and $\BFH_1^{-}$ and $\BFH_{I}^+$ are null}})}\vspace{+1.5mm}

Define matrix blocks $\BFM_i^-,\BFM_i,\BFM_i^+$ of $\BPsi^{-1}\BA^{-T}$ corresponding to the same entry indices of $\BFH_i^-,\BFH_i,\BFH_i^+$\vspace{+1.5mm}

Solve $\BFM_1$, $\BFM_1^+$\vspace{+1.5mm}

\For{$j= 2$ to $I$}{
 $\BFM_j^-=\BFM_{j-1}^+\quad $ ({{note: $\BA^\top \BPsi=\BA [K(\BT,\BT)+\sigma_y^2\bold{I}]\BA^\top $ is a symmetric matrix}}) \vspace{+1.5mm}

Solve auxiliary matrix $\BFM_j^{--}$: $$\BFH_{j-1}^{-}\BFM_{j-2}+\BFH_{j-1}\BFM_{j-1}^-+\BFH_{j-1}^+\BFM_j^{--}=0\quad (\text{note: skip for\ } j=2)$$ \vspace{+1.5mm}

 Solve $\BFM_j: \BFM_j^{--}\BFH_{j-1}^-+\BFM_j^{-}\BFH_{j-1}+\BFM_j\BFH_{j-1}^+=0$\vspace{+1.5mm}
 
 Solve $\BFM_{j}^+: \BFM_j^{-}\BFH_{j}^-+\BFM_j\BFH_{j}+\BFM_j^{+}\BFH_{j}^+=\bold{I}_{2m}\quad $ ({note: skip for $j=I$})
}

\textbf{return}: $\BFM_j^-,\BFM_j,\BFM_j^+, j=1,\cdots I$

\end{algorithm}

Algorithm \ref{alg:band_PhiA} is then  designed to compute the $m$-band of $\BPsi^{-1} \BA^{-T}$ in $\CalO(m^2n)$ time.  For notation simplicity, we let $\BPsi$ denote $\BPhi(\BT)+{\sigma_y^2}\BA$ so $\BPsi$ and $\BA$ are both $m$-banded matrices. The main concept behind Algorithm \ref{alg:band_PhiA} is that the multiplication of two $m$-banded matrices results in a $2m$-banded matrix, which can be partitioned into a block-tridiagonal matrix $\BFH = \text{diag}[\BFH^-_j, \BFH_j, \BFH^+_j]$, where each block is a $2m$-by-$2m$ matrix. Since we only require the $m$-band of $\BPsi^{-1}\BA^{-T}$, we can utilize the block-tridiagonal property of $\BFH$. This means that the multiplication of any row/column of  $\BPsi^{-1}\BA^{-T}$ by any column/row of $\BFH$ only involves three consecutive $2m$-by-$2m$ block matrices from $\BPsi^{-1}\BA^{-T}$. The process of computing the band of $\BPsi^{-1}\BA^{-T}$ is illustrated in Figure \ref{fig:banded_inverse} . Solving a $2m$-by-$2m$ matrix equation has a time complexity of $O(m^3)$, and since we only need to solve $O(nm)$ of these matrix equations, the total time complexity of Algorithm \ref{alg:band_PhiA} is $O(m^2n)$.

\begin{figure}
    \centering
    \includegraphics[width=.3\linewidth]{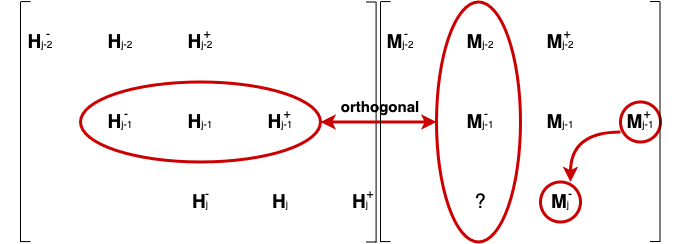}
    \includegraphics[width=.3\linewidth]{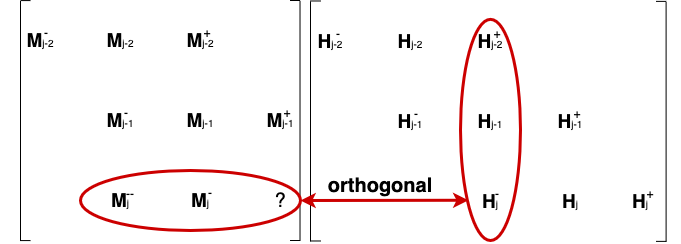}
    \includegraphics[width=.3\linewidth]{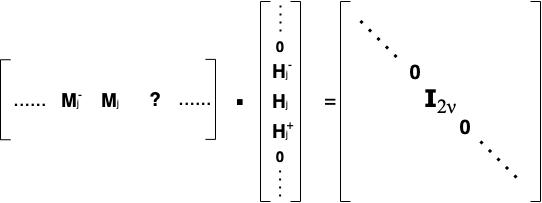}
    \caption{$\BFH$ is a block-tridiagonal matrix. When working on the $j$-th column, we can get $\BFM_{j}^-=\BFM_{j-1}^+$ directly by symmetry and solve an auxiliary matrix $\BFM_j^{--}$ by putting $[\BFM_{j-2};\BFM_{j-1}^-;\BFM_j^{--}]$ in a consecutive column (left); then we use $[\BFM_j^{--}, \BFM_j^- ,\BFM_j]$ to solve $\BFM_j$ (middle), and  $[\BFM_j^{-}, \BFM_j ,\BFM_j^+]$ to solve $\BFM_j^+$ (right). }
    \label{fig:banded_inverse}
\end{figure}

Following the training phase, the $m$-band structure of $\Psi^{-1}\BA^{-T}$ is established. To compute the posterior variance, we employ a method analogous to the one used for the posterior mean. The posterior variance at a new point $t^*$ is determined by $K(t^*,t^*)-\BPhi^\top (t^*)\Psi^{-1}\BA^{-T}\BPhi(t^*)$. Given that $\BPhi(t^*)$ contains at most $m$ non-zero consecutive entries and the $m$-band of $\Psi^{-1}\BA^{-T}$ is specified, the computational complexity for $\BPhi^\top (t^*)\Psi^{-1}\BA^{-T}\BPhi(t^*)$ is thus $\CalO(\log n)$, which accounts for the index search of non-zero entries, or even $\CalO(1)$, assuming the smallest index $i$ such that $t_i>t$ is predetermined.

If $t^*$ is predetermined, the computation of \eqref{eq:conditional-variance-KP-noisy} can be further simplified by employing banded-matrix solvers, similar to the approach used for calculating the posterior mean.

\subsubsection{Computations of log-likelihood \eqref{eq:loglike-noisy-KP} }

For learning the hyperparameter $\Btheta\in\boldsymbol{\Theta}$, we need to  directly compute the value of log-likelihood function \eqref{eq:loglike-noisy} if  $\BTheta$ is discrete or its gradient to run gradient descent if $\BTheta$ is continuous.

In the discrete scenario, calculating \eqref{eq:loglike-noisy} requires the computation of matrix inverses and determinants. The matrix inversion part  can be efficiently handled using the banded matrix solver in Section \ref{subsection:posterio_mean}. The focus now is the efficient computation of the following terms:
\[\log \det(\BPhi_{\Btheta}(\BT)+{\sigma_y^2}\BA_{\Btheta}),
    \quad\log\det(\BA_{\Btheta}).\]
Because both $\BA$ and $\BPhi_{\Btheta}(\BT)$ are $m$-banded matrices, their determinants can be computed in $\CalO(m^2n)$ time by sequential methods \citep[section 4.1]{kamgnia2014some}. 

In the continuous scenario, the gradient of log-likelihood $L$ can be written in the following form via direct calculations:
\begin{equation}
\label{eq:log_like_1}
    \begin{aligned}
    2\frac{\partial L}{\partial \theta_j}=& \text{Tr}\left(\left[K_{\Btheta}^{-1}\BZ\BZ^\top -\bold{I}\right]\BPhi_{\Btheta}(\BT)^{-1}\left(\frac{\partial\BPhi_{\Btheta}(\BT)}{\partial \theta_j}-\frac{\partial \BA_{\Btheta}}{\partial \theta_j}K_{\Btheta}\right)\right)\\
    =&\underbrace{\text{Tr}\left(\left[\BZ^\top \BPhi_{\Btheta}(\BT)^{-1}\frac{\partial\BPhi_{\Btheta}(\BT)}{\partial \theta_j}\right]\left[\BA_{\Btheta}\BPhi_{\Btheta}(\BT)^{-T}\BZ\right]\right)}_{A}\\
    &-\underbrace{\text{Tr}\left(\left[\BZ^\top \BPhi_{\Btheta}(\BT)^{-1}\right]\left[\frac{\partial \BA_{\Btheta}}{\partial \theta_j}\BZ\right]\right)}_{B}-\underbrace{\text{Tr}\left(\BPhi_{\Btheta}(\BT)^{-1}\frac{\partial\BPhi_{\Btheta}(\BT)}{\partial \theta_j}\right)}_{C}+\underbrace{\text{Tr}\left(\BA^{-1}_{\Btheta}\frac{\partial \BA_{\Btheta}}{\partial \theta_j}\right)}_{D}
\end{aligned}
\end{equation}
where the second equality is from the KP identities $K_{\Btheta}=\BPhi_{\Btheta}(\BT)^{T}\BA_{\Btheta}^{-1}=\BA_{\Btheta}^{-1}\BPhi_{\Btheta}(\BT)$ and the last equality is from the identity $\text{Tr}(AB)=\text{Tr}(BA)$ for any matrices $A$ and $B$.

Notice that both terms  $A$ and $B$ in \eqref{eq:log_like_1} are  scalars:
\begin{align*}
    &\text{Tr}\left(\left[\BZ^\top \BPhi_{\Btheta}(\BT)^{-1}\frac{\partial\BPhi_{\Btheta}(\BT)}{\partial \theta_j}\right]\left[\BA_{\Btheta}\BPhi_{\Btheta}(\BT)^{-T}\BZ\right]\right)=\left[\BZ^\top \BPhi_{\Btheta}(\BT)^{-1}\frac{\partial\BPhi_{\Btheta}(\BT)}{\partial \theta_j}\right]\left[\BA_{\Btheta}\BPhi_{\Btheta}(\BT)^{-T}\BZ\right],\\
    &\text{Tr}\left(\left[\BZ^\top \BPhi_{\Btheta}(\BT)^{-1}\right]\left[\frac{\partial \BA_{\Btheta}}{\partial \theta_j}\BZ\right]\right)=\left[\BZ^\top \BPhi_{\Btheta}(\BT)^{-1}\right]\left[\frac{\partial \BA_{\Btheta}}{\partial \theta_j}\BZ\right].
\end{align*}
Because $\BPhi_{\Btheta}(\BT)$, $\BA_{\Btheta}$, $\frac{\partial \BA_{\Btheta}}{\partial \theta_j}$, and $\frac{\partial\BPhi_{\Btheta}(\BT)}{\partial \theta_j}$ are all banded matrices, terms  $A$ and $B$ can be computed in $\CalO(m^3n)$ time using banded matrix solver as described in Section \ref{subsection:posterio_mean}.

For the computation of terms $C$ and $D$ in \eqref{eq:log_like_1}, both formulated as $\text{Tr}(A^{-1}B)$ with $A$ and $B$ being $m$-banded matrices, the approach outlined in Section \ref{subsec:posterior_variance} is applicable. The focus is on computing the $m$-band of $A^{-1}$, given that $B$ is $m$-banded, which implies the computation of trace involves only the $m$-band of $A^{-1}$. This process is achievable by simply replacing the $2m$ in \eqref{eq:tridiagonal_block} and Algorithm \ref{alg:band_PhiA} by $m$ and can be finished within $\CalO(m^3n)$ time. With the $m$-band of $A^{-1}$ identified, the calculation of $\text{Tr}(A^{-1}B)$ is similarly efficient, maintaining the overall time complexity at $\CalO(m^3n)$.

\subsection{Multi-dimensional Gaussian Processes}
\cite{chen2022kernel} introduced efficient training and prediction algorithms for GPs using KPs of Matérn-type kernels when observations are from grid-based designs. These algorithms are also applicable to KPs with general kernels, as their key idea is to covert the kernel matrix by a Kronecker product of sparse banded matrices, which can be constructed by taking the Kronecker product of the outputs from Algorithm~\ref{alg:KP} for general kernels.

We propose  algorithms for GPs' efficient training and prediction using general KPs of product kernel $K(\Bt,\Bt)=\prod_{d=1}^DK_d(t_d,t_d)$ or additive kernels $K(\Bt,\Bt)=\sum_{d=1}^DK_d(t_d,t_d)$ when observations are scattered . We suppose  each $K_d$ is the kernel function of a one-dimensional GP following parametrized SDE \eqref{eq:SODE_parametrize} meeting Condition \ref{condition:full_rank}.

In multiple dimensions, the banded structure of KPs no longer holds. However, we can still use KP to decompose the kernel matrix $K(\BT,\BT)$ into sparse matrices, thereby enabling efficient computation of \eqref{eq:conditional-mean-noisy}–\eqref{eq:loglike-noisy}. We first propose the following algorithm for such a decomposition:
\begin{algorithm}[h]

\textbf{Input: }{ Scattered points $\BT$ , product kernel $K=\prod_d K_d$ or additive kernel $K=\sum_dK_d$} 

\textbf{Return: }{sparse matrices $\textbf{A}$ and kernel packets $\BPsi(\cdot)$} 


\For{$i=1,2,\cdots, n$}{  
 Search for the $s$ nearest points to  $\mathbf{t}_i$  in $\mathbf{T}$, denoted by $\{\mathbf{t}_{i_j}\}_{j=1}^s$, where $ s = (2m)^D + 1$  for  product kernel and  $s = 2mD + 1$  for additive kernel;\vspace{+1.5mm}

Delete $\Bt_i$ ;  \small{(Note: This is to ensure that the same nearest point set  $\{\mathbf{t}_{i_j}\}_{j=1}^s$ will not appear in any following iteration, otherwise, the resulted $\BA$ is not invertible)}\vspace{+1.5mm}

Solve for  coefficients $\sum_{j=1}^sa_j H(\Bt_{i_j})=0$ according to Theorem \ref{thm:combine_ker_prod_multi_dim} for product kernel and Theorem \ref{thm:combine_ker_add_multi_dim} for additive kernel, then assign $\BA_{\Bt_i,\Bt_{i_j}}=a_j$ for $j=1,\cdots,s$\  ;
} 
$\BPsi(\cdot)=\BA K(\BT,\cdot)$ \ \small{(Note: $\BA_{\Bt,\Bt'}$ has the same index as $K(\Bt,\Bt')$ in $K(\BT,\BT)$.)}
\caption{Computing sparse transformation matrix $\textbf{A}$ and kernel packets $\BPsi(\cdot)$} \label{alg:KP_multi_dim}

\end{algorithm}

In Algorithm \ref{alg:KP_multi_dim}, the search for nearest neighbors can be achieved in $\CalO(1)$ time for structured samples (e.g. partitioned samples), and in $\CalO(\log n)$ time in the worst case by using the matching algorithm in \cite{friedman1977algorithm}, which is implemented as a built-in function in MATLAB. Therefore, the time complexity of Algorithm \ref{alg:KP_multi_dim} is $\CalO(n\log n)$.  Figure~\ref{fig:KP_muli_dim} illustrates two-dimensional KPs constructed from 1,000 scattered points on $[0,1]^2$ for product kernel $K(\Bt,\Bt')=e^{-|t_1-t_1'|-|t_2-t_2'|}$. Ten KP functions are randomly selected from $\BPsi(\cdot)$ (i.e., ten rows of of $\BPsi(\cdot)$) and plotted. Each KP function is compactly supported on a small region, making $\BPsi(\BT)$ a sparse matrix.

 \begin{figure}
    \centering
    \includegraphics[width=.35\linewidth]{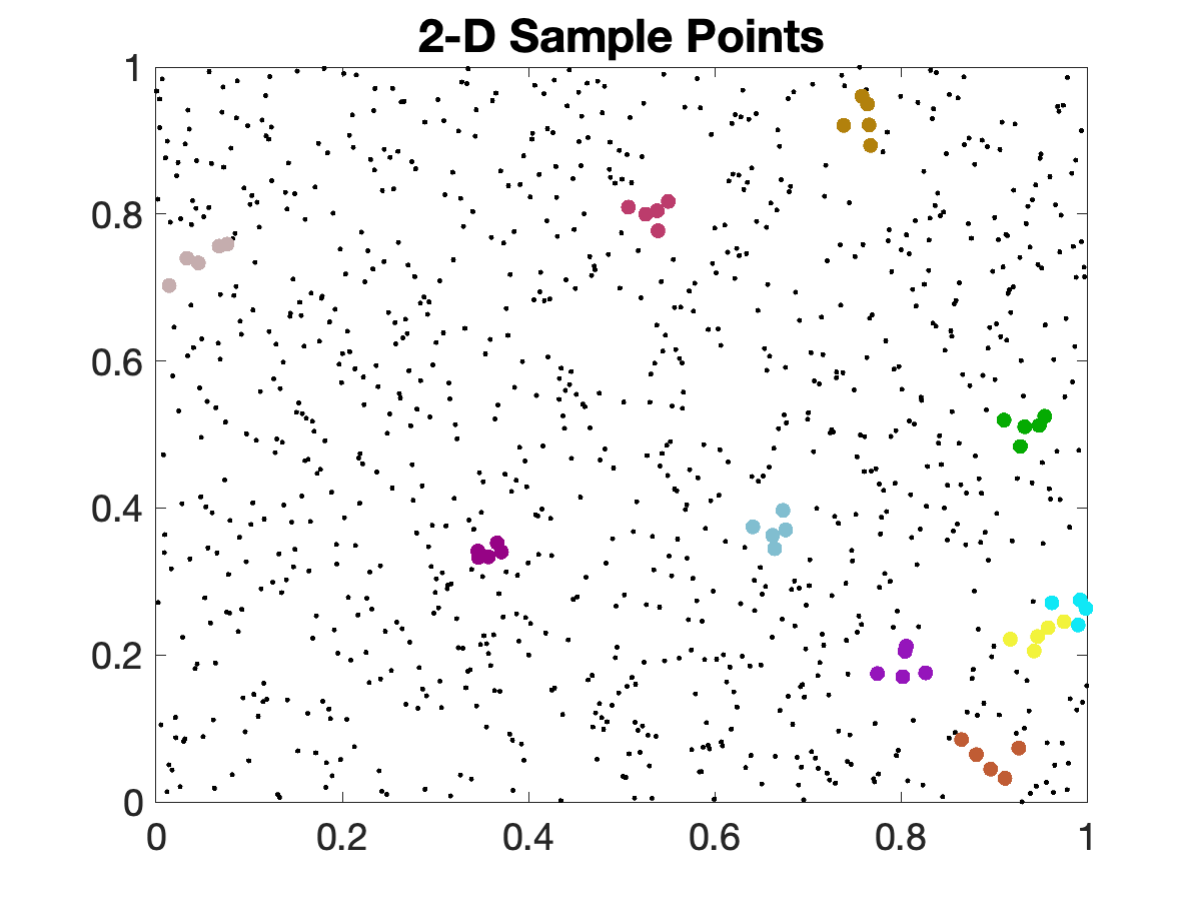}
     \includegraphics[width=.35\linewidth]{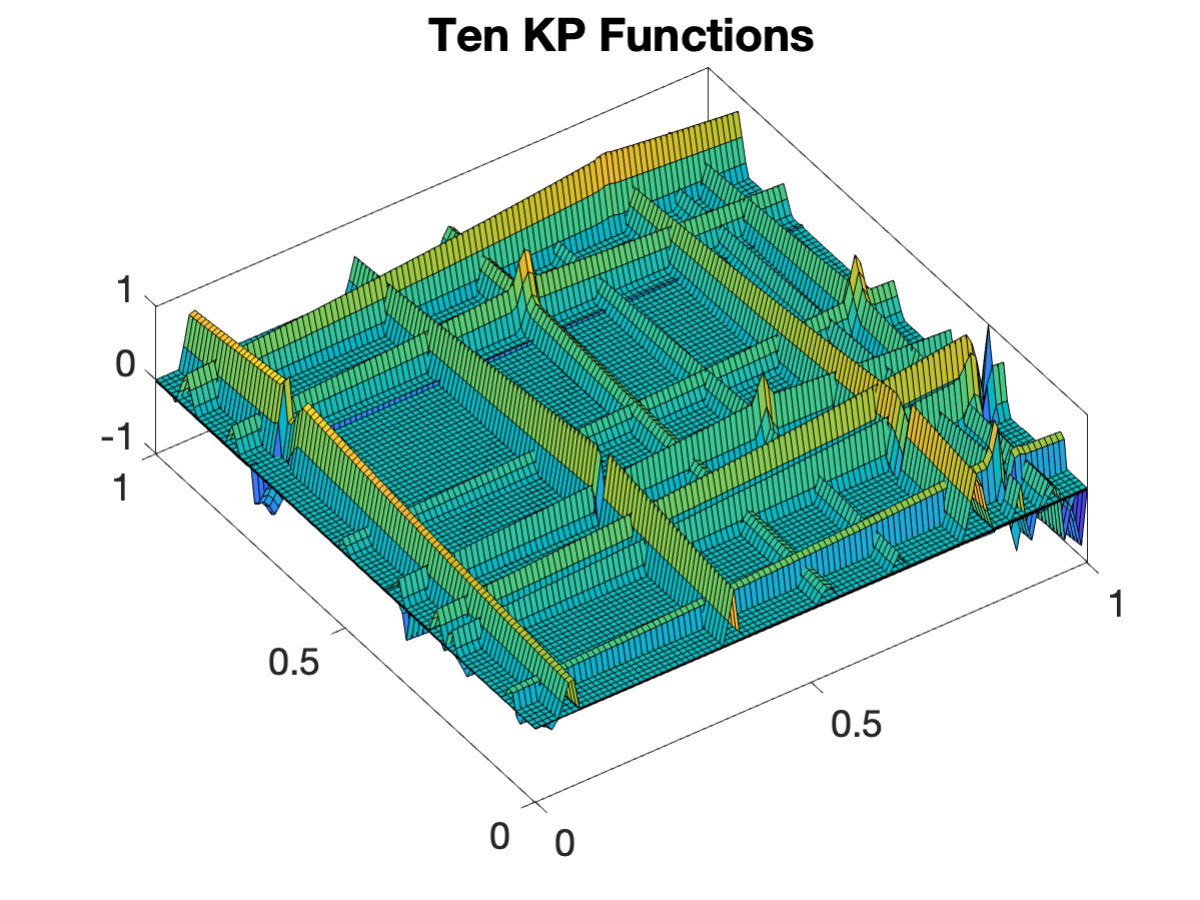}
    
    \caption{Left: Ten groups of points in different colors are selected from 1,000 sample points to construct ten KP functions. Right: The plots of the ten KP functions for product kernel $K(\Bt,\Bt')=e^{-|t_1-t_1'|-|t_2-t_2'|}$ corresponding to the selected groups of points.}
    \label{fig:KP_muli_dim}
\end{figure}

Moreover, using the sparse decomposition $\BA K(\BT,\BT)=\Psi(\BT)$, \eqref{eq:conditional-mean-noisy} and \eqref{eq:conditional-variance-noisy} can be efficiently computed via iterative methods, such as the conjugate gradient method \cite{wang2019exact} or kernel gradient \cite{ding2024random}. These iterative methods require $\CalO(\log n)$ iterations, with each iteration involves matrix-vector multiplications only. Consequently, the sparse decomposition provided by KPs significantly improves computational efficiency by allowing sparse matrix multiplications.

\section{Numerical Experiments}
\label{sec:experiments}
To evaluate the performance of KPs, we apply our algorithms to additive GPs and GPs with product-form kernels across several datasets, each containing millions of data points, and compare their prediction errors on millions of test samples. Due to the massive size of the test sets, classical SDE-based GP algorithms fail because of excessive memory requirements. In contrast, KPs offer a feasible solution without resorting to any low-rank approximations because of their inherent sparsity. We select the following two low-rank approximation algorithms as benchmarks:
\begin{enumerate}  [nosep]
    \item \textbf{Random Fourier feature (RFF)}:  Approximates the Gaussian kernel using Fourier features $\{\cos(\boldsymbol{\omega}_j^\top (\mathbf{t} - \mathbf{t}'))\}_{j=1}^m$, where $\{\boldsymbol{\omega}_j\}_{j=1}^m$ are i.i.d. samples drawn from a spherical Gaussian distribution \citep{rahimi2008random}. We set $m = 1000$, which achieves a good balance between numerical accuracy and computational efficiency according to our experiments.
    \item \textbf{Sparse GP}: Randomly select $m$ points $\{\tilde{\Bt}_j\}_{j=1}^m$ from the data sets and use $\{K(\tilde{\Bt}_j,\cdot)\}_{j=1}^m$ as basis functions to approximate the GP \citep{SnelsonGhahramani05}. We also set $m = 1000$ for a good balance between numerical accuracy and computational efficiency.
\end{enumerate}

\subsection{Additive GPs}
We evaluate KPs on the additive GP model using the SUSY dataset, which contains five million samples, each with $D = 8$ dimensions, and is designed to classify whether an observed event originates from a supersymmetric signal process. We use four million samples for training and the remaining one million for testing. We use the probit approximation $\sigma(y(\Bt))$ \citep[Section 3]{RasmussenWilliams06} to approximate the conditional  binary distribution of data where $y(\Bt)$ is a 8-dimensional additive GPs
\begin{equation}
\label{eq:additive_susy}
    y(\Bt)=\sum_{d=1}^8y^{(d)}(t_d)
\end{equation}
The conditional distribution $y(\mathbf{T}^*) \mid \mathbf{T}$ at the test points $\mathbf{T}^*$, given the training points $\mathbf{T}$, can be computed using the backfitting algorithm \citep{saatcci2012scalable}, which iteratively solves a series of one-dimensional GP regressions for $\{y^{(d)}\}_{d=1}^8$. Our experiments are conducted with training set sizes of one, two, three and four million, and are evaluated on one million randomly selected test points.

For KPs and sparse GPs, we use the Mat\'ern-$3/2$ kernel $K_{\rm mat}$ for each additive components, i.e. the  $y^{(d)}\sim \CalN(0,K_{\rm mat})$. Since KPs are not subject to any approximation, their error rate decreases significantly as the data size increases. In contrast, for RFF and Sparse GP, because their approximation capacity is constrained by the approximation degree $m$, their performance improvement with increasing training data size is much less obvious.
\begin{table}[ht]

\centering
\caption{Classification error rate  of SUSY}
\vspace{5pt}
\scalebox{0.85}{ 
\begin{tabular}{ |p{3cm}||p{3cm}|p{3cm}|p{3cm}|p{3cm}|  }
 \hline
Data size & one million & two million & three million& four million\\
 \hline
KPs     &23.86\%&  22.25\% & 20.14\%& 18.89\%\\
 \hline
RFF  &  33.93\%   &33.27\% &  33.03\%   &32.93\%\\
\hline
Sparse GP & 24.16\% & 23.86\% & 23.64\% & 23.44\%\\
\hline
\end{tabular}}

\label{tab:SUSY}
\end{table}

\subsection{Product Form Kernels}
 \begin{figure}
    \centering
    \includegraphics[width=.32\linewidth]{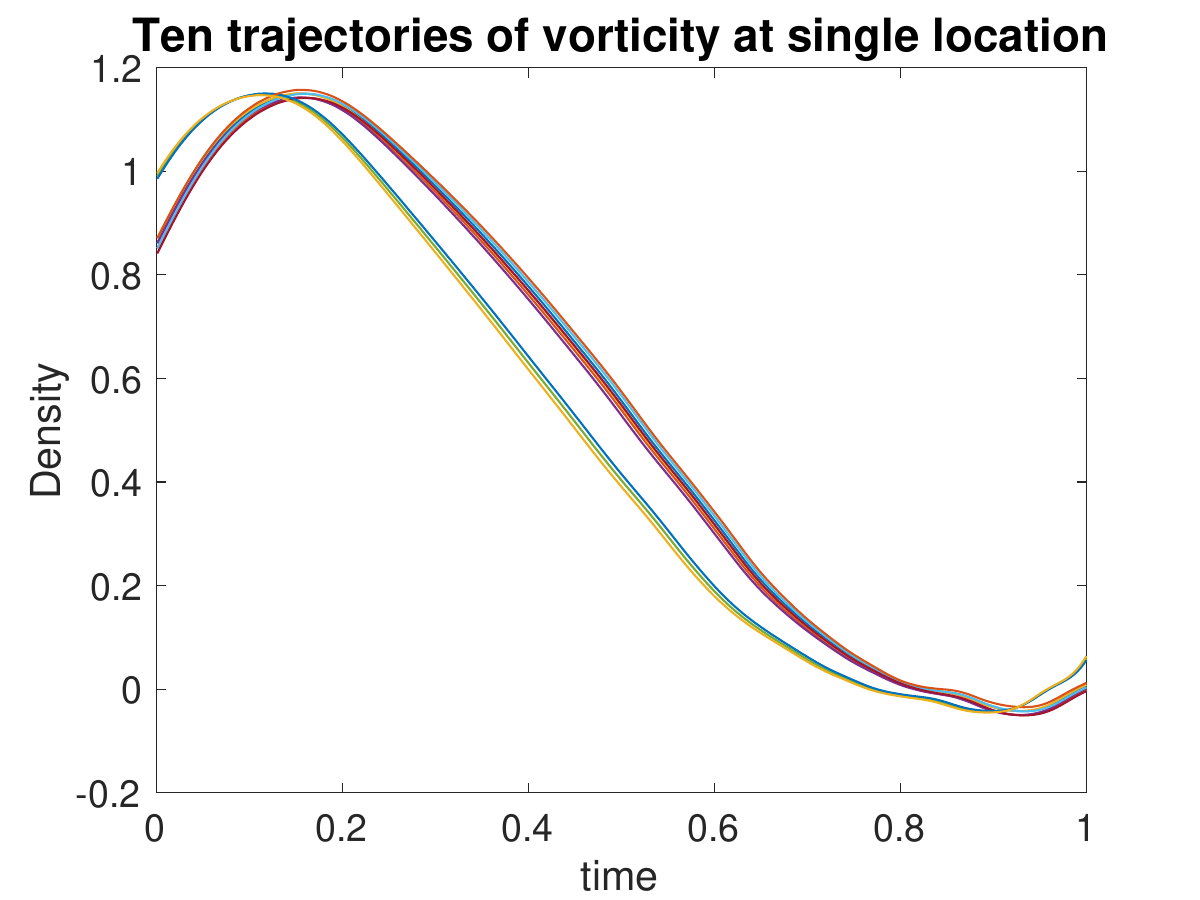}
    \includegraphics[width=.32\linewidth]{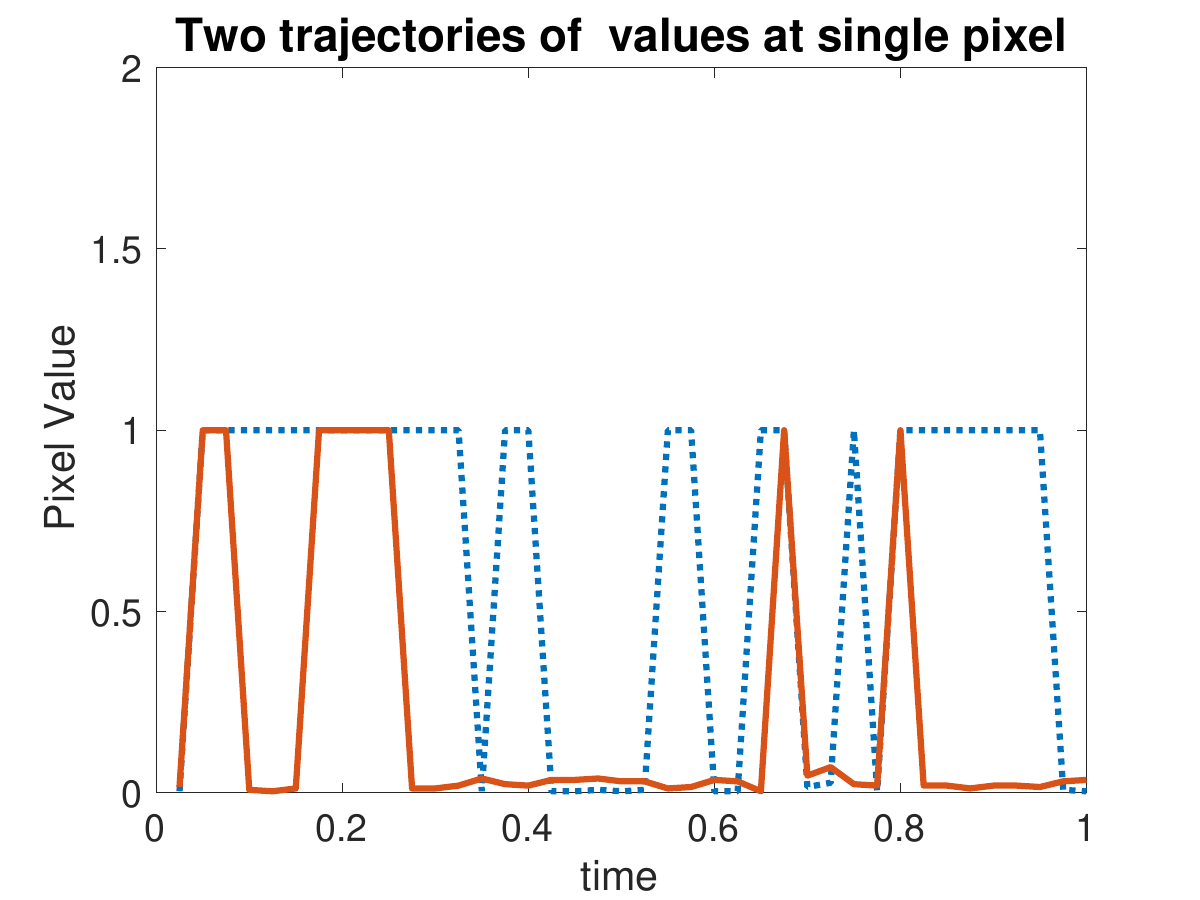}
    \includegraphics[width=.32\linewidth]{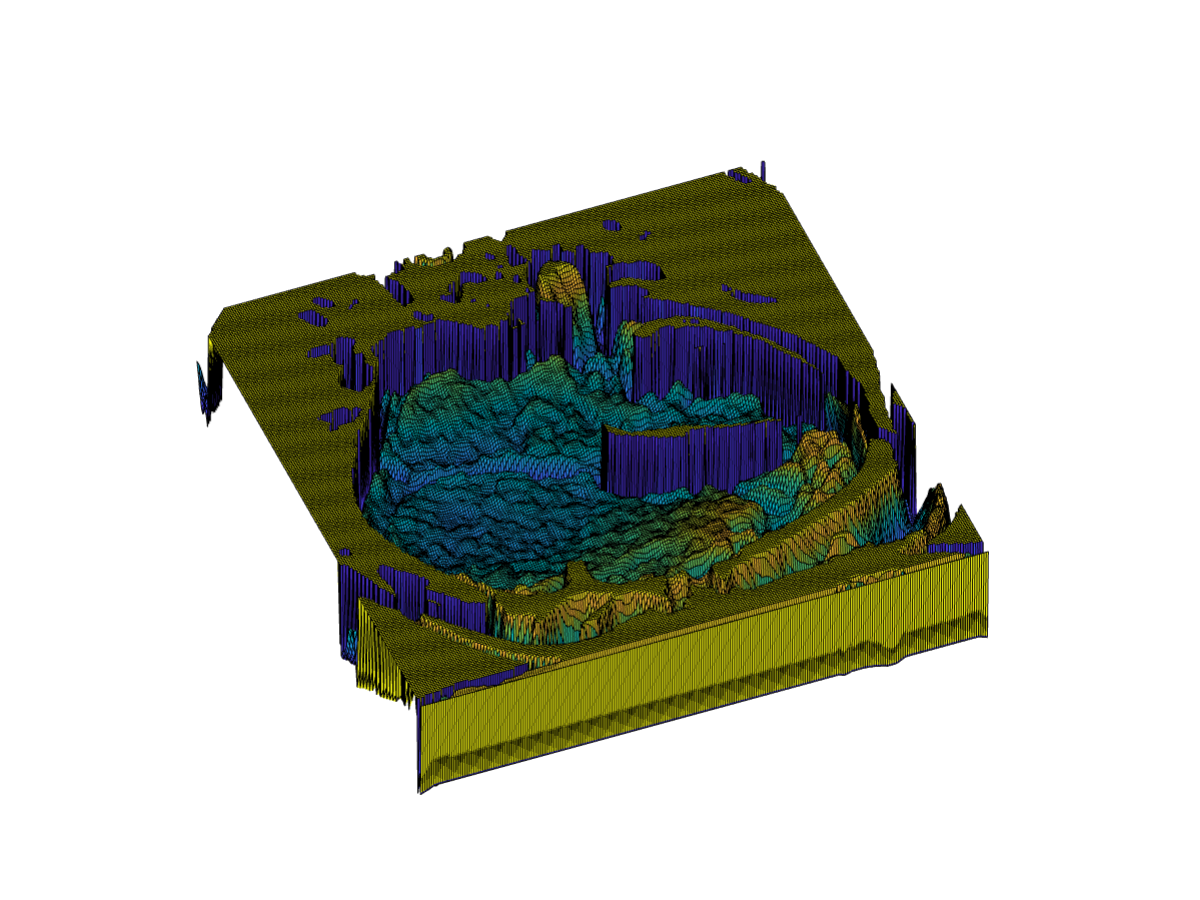}

    \caption{Left: Trajectories of vorticity at different locations; Middle: Trajectories of values at two different pixels; Right: MRI is treated as a non-smooth function on surface.}
    \label{fig:experiment_surface}
\end{figure}
\subsubsection{Euler Flow:} 
We solve the trajectory of a two-dimensional vorticity governed by the Euler equation. Specifically, we randomly select 20, 40, 60, 80, and 100 time slices of the vorticity dynamics. For each time slice, we use 10,000 scattered triangular finite elements  and solve the weak solution of the Euler equation constructed by these finite elements. Consequently, the total data sizes in the experiments are $n = 2\times10^5$, $4\times10^5$, $6\times10^5$, $8\times10^5$, and $10^6$ scattered points, respectively. The two spatial dimensions and one temporal dimension form a three-dimensional PDE reconstruction problem. Thus, the vorticity dynamics are modeled as a GP sample path with the following product kernel.
\begin{equation}
    \label{eq:vortex_kernel}
    K((\Bx,t),(\Bx',t'))=[\sin(2\pi |t-t'|)+\cos(2\pi|t-t'|]\exp\{-|t-t'|\} \exp\{-\|\Bx-\Bx'\|_1\}
\end{equation}
where the temporal dimension is modeled using the differentiable and periodic kernel $[\sin(2\pi |t-t'|)+\cos(2\pi|t-t'|]\exp\{-|t-t'|\}$ since the vorticity at a fixed spatial point exhibits a single-period pattern, as shown in the first plot of Figure \ref{fig:experiment_surface}. For the spatial dimensions, we use the product Laplace kernel $\exp\{-\|\Bx-\Bx'\|_1\}=\exp\{-|x_1-x_1'|-|x_2-x_2'|\}$, which possesses the spatial Markov property \citep{ding2024sample} and captures the second-order spatial relation governed by the Euler equation.

The experimental results are presented in the first column of Figure \ref{fig:product_form_exp}. The MSE comparison shows that the MSE of KPs consistently decreases as the sample size increases, since KPs perform exact computations without relying on any approximation. In contrast, the two competing methods exhibit little improvement with increasing sample size due to errors introduced by their low-rank approximations. As shown in the plots, KPs accurately capture the spatiotemporal vorticity dynamics.

 \begin{figure}
    \centering
    \includegraphics[width=.22\linewidth]{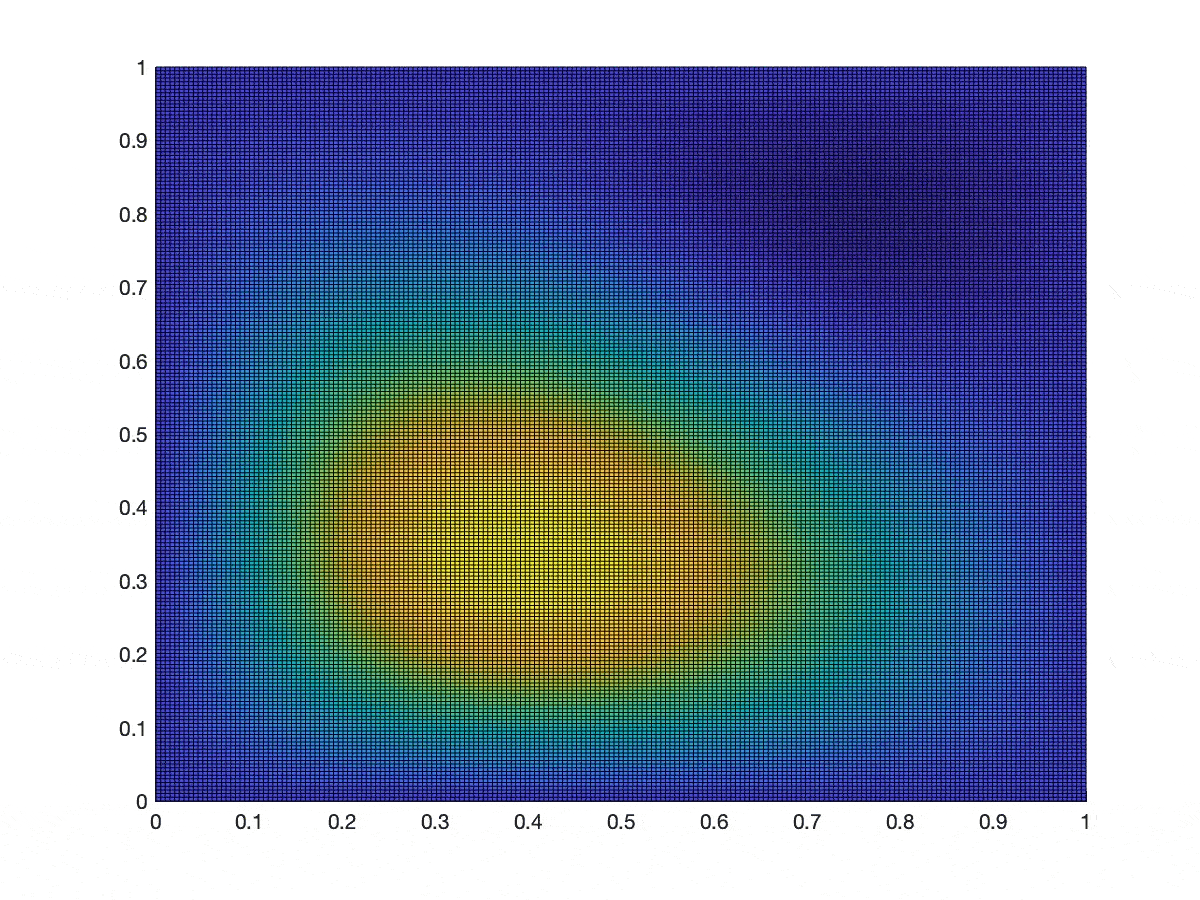}
    \includegraphics[width=.22\linewidth]{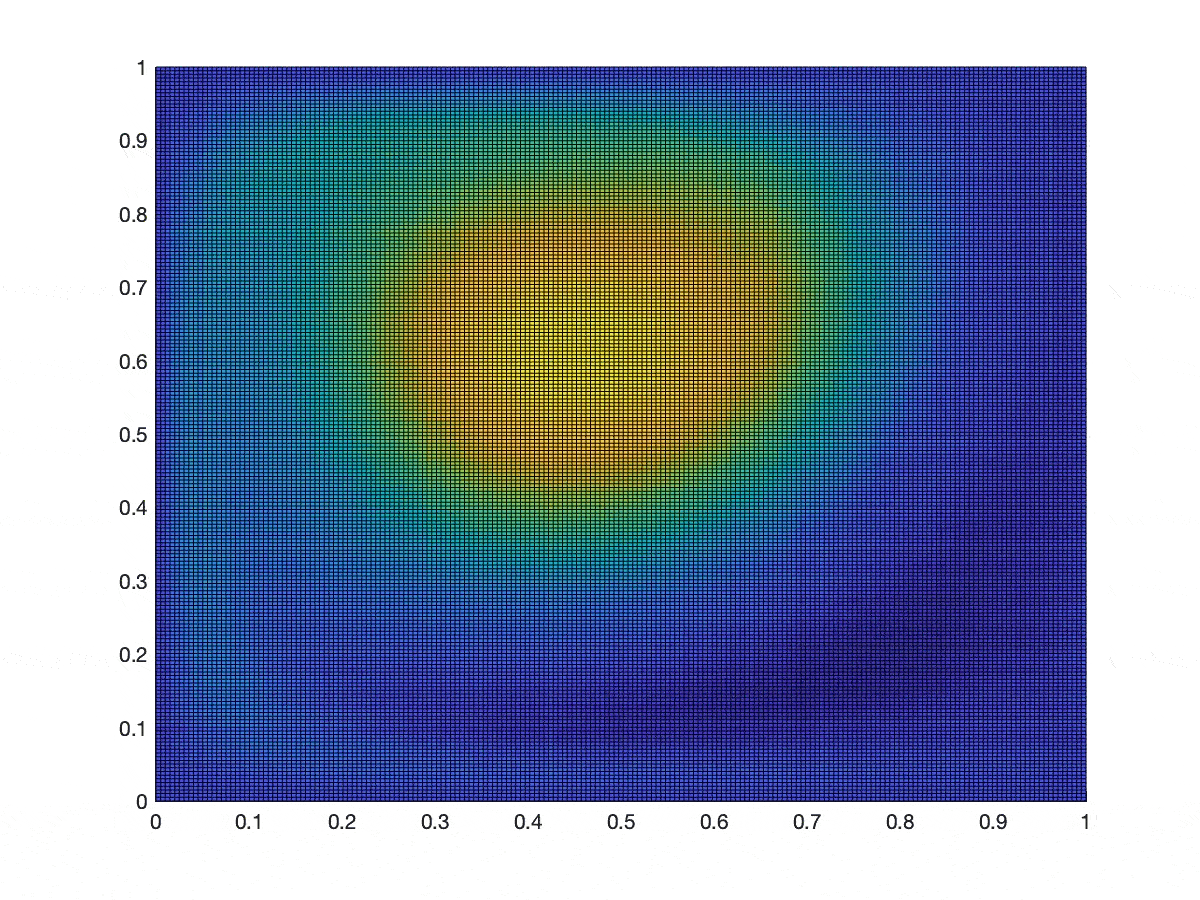}
    \includegraphics[width=.22\linewidth]{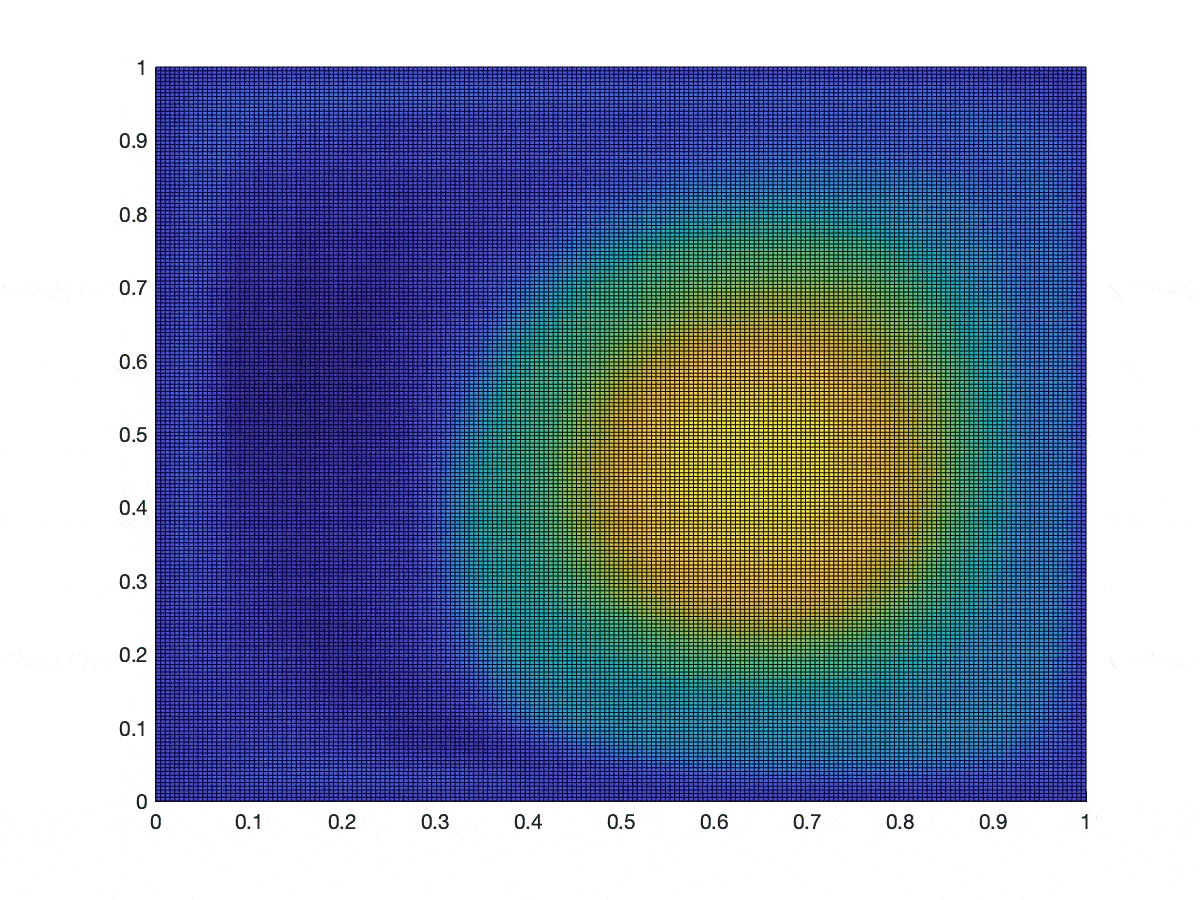}
    \includegraphics[width=.22\linewidth]{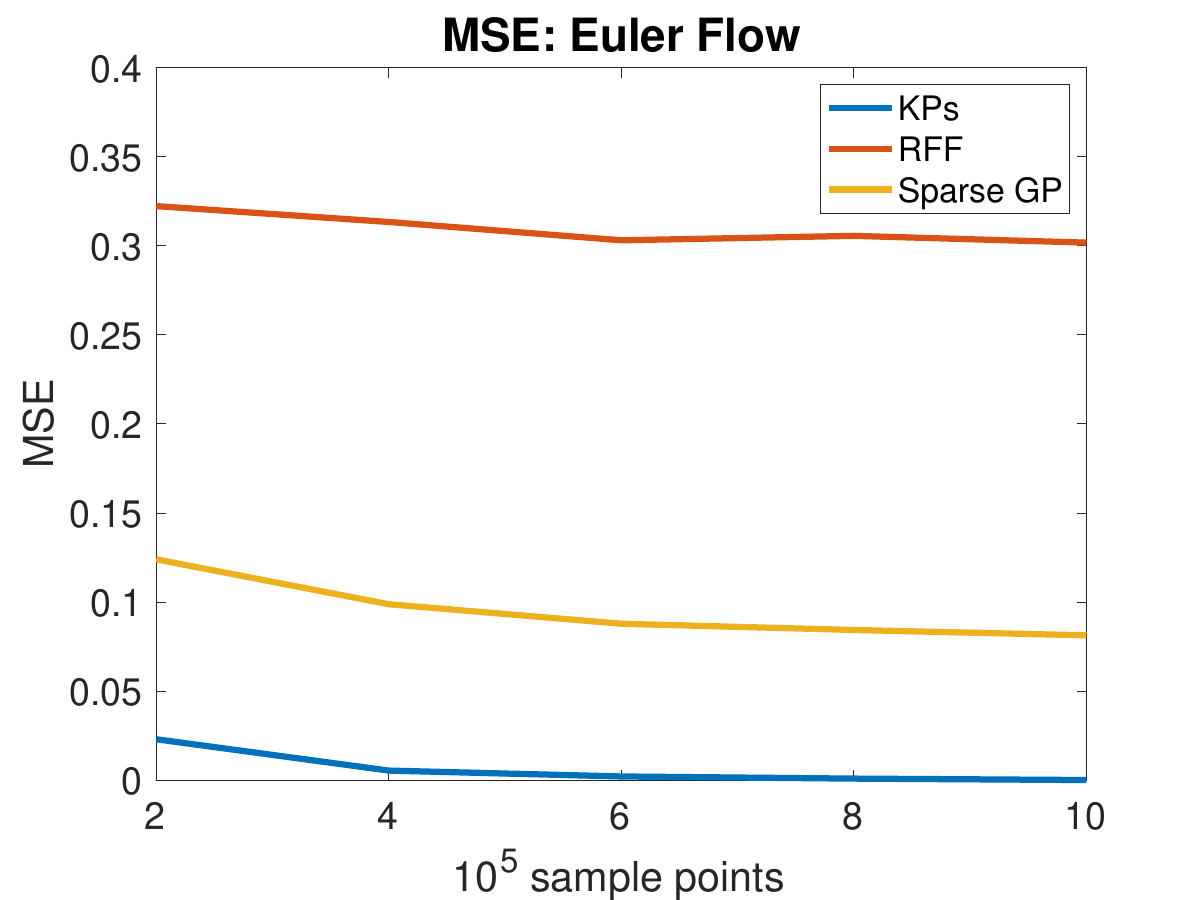}
    \includegraphics[width=.22\linewidth]{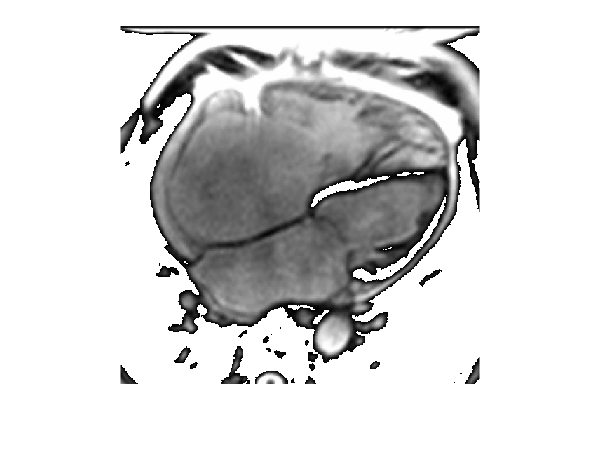}
    \includegraphics[width=.22\linewidth]{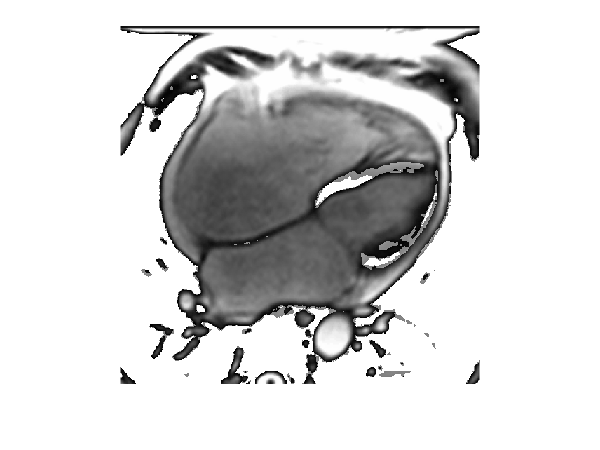}
    \includegraphics[width=.22\linewidth]{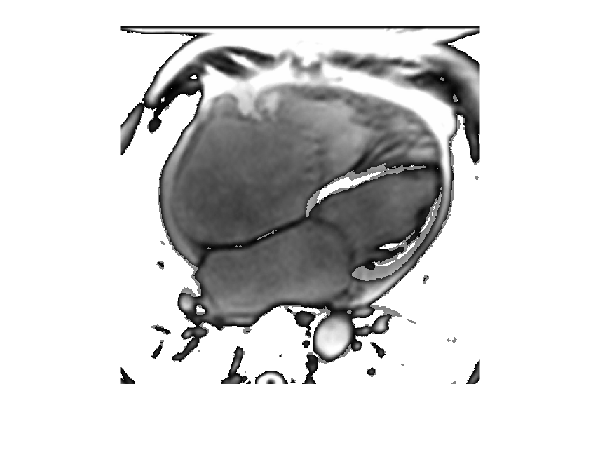}
    \includegraphics[width=.22\linewidth]{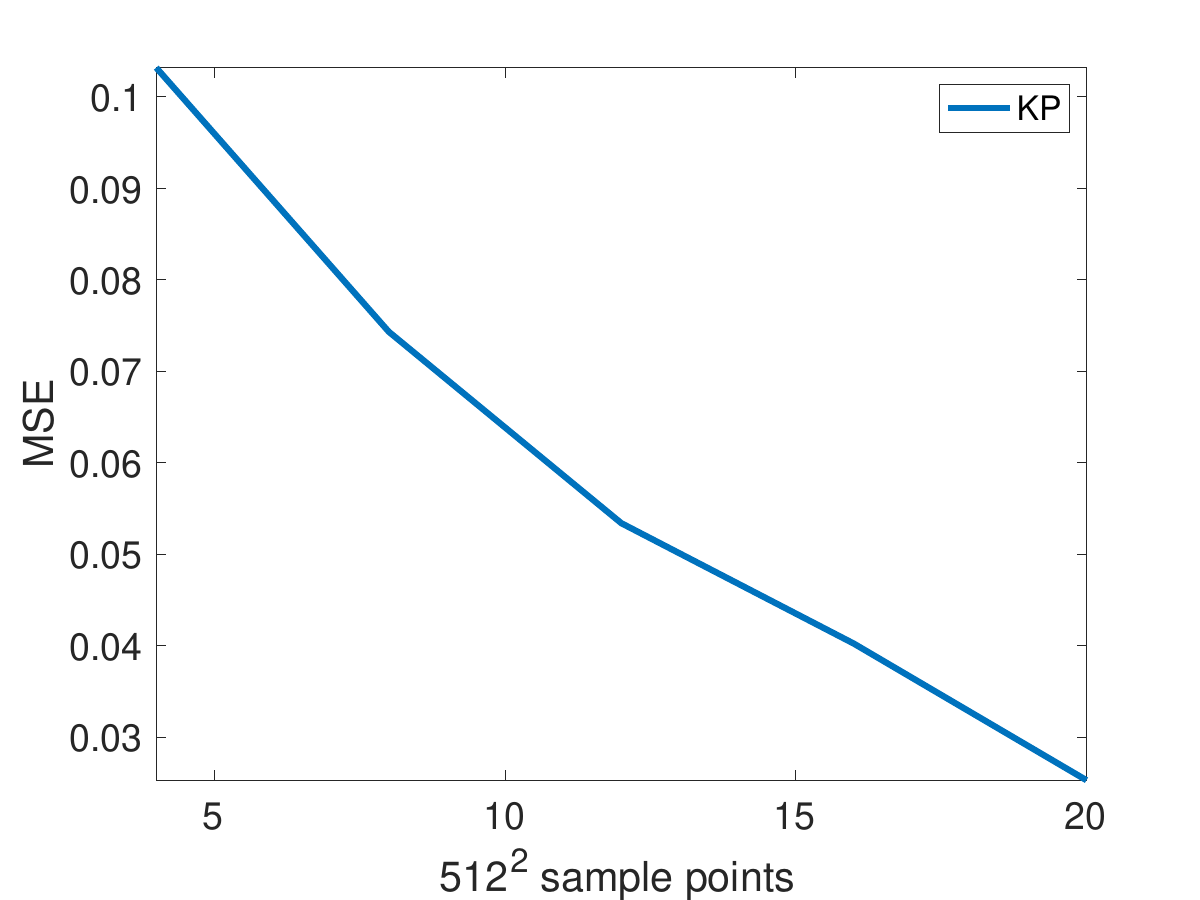}
    
    \caption{ Upper row: time trajectory of the vortex dynamics interpolated by KPs and MSE of competing algorithms; Lower row:  time trajectory of the real-time cardiac MRI interpolated by KPs and MSE of KPs}
    \label{fig:product_form_exp}
\end{figure}

\subsubsection{Real-Time MRI}

We use a sequence of $512 \times 512$ cardiac MRIs from a real-time MRI video as training data. Each MRI frame is treated as a function on the domain $[0,1]^2$, and the entire video as a spatiotemporal function on $[0,1]^2 \times [0,T]$. For each frame, we use its $256 \times 256$ pixels as training samples and select 5, 10, 15, and 20 frames from the video for reconstruction. Consequently, the total data sizes in the experiments are $n=5\times 256^2$, $10\times 256^2$, $15\times 256^2$, and $20\times 256^2$ gridded points, respectively. Similar to the previous experiment, this is a three-dimensional dynamics reconstruction problem and we model the dynamics as a GP sample path with the following product kernel:
\begin{equation}
    \label{eq:MRI_kernel}
    K((\Bx,t),(\Bx',t'))=\cos(10\pi|t-t'|]\exp\{-|t-t'|\} \exp\{-10\|\Bx-\Bx'\|_1\}
\end{equation}
where the temporal dimension is modeled using the non-differentiable and periodic kernel $\cos(10\pi|t-t'|]\exp\{-|t-t'|\} $ since each pixel exhibits a ten-period impulse pattern, as shown in the second plot of Figure \ref{fig:experiment_surface}. For the spatial dimensions, we use the product
Laplace kernel $\exp\{-\omega|t-t'|\} \exp\{-10\|\Bx-\Bx'\|_1$ with wavelength $\omega=10$ to model the highly non-smooth MRI surfaces as shown in the right plot in Figure \ref{fig:experiment_surface}.

Due to the large data volume and the highly non-smooth nature of MRIs, RFF and Sparse GP fail to yield meaningful approximations. In contrast, KPs still produce accurate reconstructions of the real-time MRI, as shown in the second row of Figure \ref{fig:product_form_exp}. This is because KPs perform exact computations, which are essential for accurately reconstructing non-smooth functions in this experiment.

\section{Conclusion}
\label{sec:conclusion}

{{In this study, we develop a general theory for constructing KPs for a broad class of GPs driven by SDEs. We further propose an exact and efficient algorithm that derives both forward and backward KPs and combines them to obtain compactly supported kernel representations. This algorithm enables $\mathcal{O}(n)$ training and $\mathcal{O}(\log n)$ or even $\mathcal{O}(1)$ prediction, while maintaining the exactness of GP inference. The KP framework also generalizes beyond the state space setting to handle scattered and multi-dimensional inputs without relying on low-rank approximations. Extensive experiments confirm that KPs achieve scalable and memory-efficient inference on large-scale additive and product-form GPs, outperforming existing SDE-based and approximate methods.  }}

\appendix

\section{Technical Details of KPs Associated with State Space Models}
\label{sec:proof of ss}

We combine the forward and backward SS models to prove Theorem \ref{thm:fundamental-Covariance_centalKP} and main theorem.

\subsection{Forward stochastic differential equation}
We first show that $m+1$ equations are enough to determine a right-sided KP system. Under condition \ref{condition:full_rank}, we  rewrite \eqref{eq:SODE} using P\'olya factorization \citep[Theorem 4.59]{bohner2001dynamic}.
\begin{theorem}[P\'olya Factorization]
\label{thm:product_form}
   For fundamental solutions $h_1,\cdots,h_{j+1}$, define
\[
 \CalW_{j+1}[h_1,\cdots,h_{j+1}](t) = \text{det}
\begin{bmatrix}
h_1(t) & \cdots&h_{j+1}(t)\\
\vdots& &\vdots\\
\frac{\partial^{j}}{\partial t^{j}}h_1(t) & \cdots& \frac{\partial^{j}}{\partial t^{j}} h_{j+1}(t)
\end{bmatrix}.\]
Under Condition \ref{condition:full_rank}, $0<|\CalW_j|<\infty$ for each $j$. Then \eqref{eq:SODE} has the  equivalent product form:
\begin{align}
    \CalL[y]=\frac{1}{u_{m+1}}\frac{\partial}{\partial t}\tilde{D}_m\tilde{D}_{m-1}\cdots \tilde{D}_3\tilde{D}_2 \frac{y }{u_1}
    =W \label{eq:product_form}
\end{align}
where functions $u_1=\CalW_1$, $1/u_2=\CalW_1^2/\CalW_2$, $1/u_{m+1}=\CalW_{m-1}/\CalW_m$, $1/u_j=\CalW_{j-1}^2/(\CalW_{j}\CalW_{j-2})$ for $j=3,\cdots,m-1$, and the differential operator $\tilde{D}_j$ is defined as
$\tilde{D}_j=\frac{1}{u_j}\frac{\partial}{\partial t}$.
\end{theorem}
There are several benefits of the product form \eqref{eq:product_form}. Firstly, the existence of a set of fundamental solutions denoted as $\{P_j\}_{j=1}^m$ is assured, with each $P_{j}$ being an $j$-th differentiable function:
\begin{equation}
    \label{eq:P_j}\tilde{D}_{j+1}\cdots\tilde{D}_2P_{j}=0,\quad \tilde{D}_{j}\cdots\tilde{D}_2P_{j}=1
\end{equation}
where each $\tilde{D}_i=\frac{1}{u_i}\frac{\partial}{\partial t}$  is a generalized first order derivative with $0<|u_i|<\infty$.

Secondly, the SDE \eqref{eq:product_form} can be written as a first-order $m$-dimensional Markov process:
\begin{align}\label{eq:forward_Markov}
 \left\{
\begin{array}{ll}
\partial_t z^*(t)=F^*(t)z^*(t)+LW(t)  \\[0.5ex]
y(t)=Hz^*(t)
\end{array}
\right.,\quad t\in(t_0,T)
\end{align}
where
\begin{equation*}
    F^*(t)=\begin{bmatrix}
0& u_2(t) & 0 &\cdots & 0\\
& & \vdots & &\\
0&0&0& \cdots&u_{m}(t)\\
0& 0&0&\cdots &0
\end{bmatrix},
\end{equation*}
 $z^*=[y,\tilde{D_2}y,\cdots]$ is a vector with $j$-th entries $z^*_{j}=\tilde{D}_jz^*_{j-1}$. 

Using forward SDE \eqref{eq:forward_Markov}, we can check that each fundamental solution $P_j$ satisfies
\begin{equation}
\label{eq:first_order_Pi}
    \partial_t [P_j, \tilde{D}_2P_j,\cdots,(\tilde{D}_m\cdots\tilde{D}_2)P_j]^\top =F(t)[P_j, \tilde{D}_2P_j,\cdots,(\tilde{D}_m\cdots\tilde{D}_2)P_j]^\top.
\end{equation}
This identity, coupled with the covariance equation $R^*(t,\mu)=\E z^*(t)z^{*\top}(\mu)=e^{\int_\mu^t F^*(\tau)d\tau}R^*(\mu,\mu)$ for $t\geq \mu$ \citep[(2.34)]{solin2016stochastic}, results in a specific right-KP equations for $R^*$ as follows:

\begin{theorem}
\label{thm:fundamental-Covariance_rightKP}
     Suppose Condition \ref{condition:full_rank} holds.  For any consecutive points $t_1<\cdots<t_{m+1}\in(t_0,T)$, 
    \[\sum_{j=1}^{m+1}a_j
    [
        P_1(t_j)\ 
        \cdots\ 
        P_m(t_j)
    ]^\top =0\quad\text{if and only if}\quad \sum_{j=1}^{m+1}a_j R^*(t_1,t_j)=0 \]
    where $[a_j]_{j=1}^{m+1}$ is one-dimensional. Therefore, $s=m+1$ for an irreducible right-sided KP.
\end{theorem}
{{
\begin{remark}
   Here,  ``one-dimensional" signifies uniqueness up to a scalar. This theorem establishes that the minimum $s$ for right-KP is $s=m+1$. 
\end{remark}}}

\begin{proof}[Proof of Theorem \ref{thm:fundamental-Covariance_rightKP}]

(1) Suppose $$\sum_{j=1}^{m+1}a_j \left[P_1(t_j),\cdots, P_m(t_j)\right]^\top =0.$$ From the differential form for each $P_i$, it is obvious that $\{P_i\}$ is a set of linearly independent fundamental solution and, as a result, the solution $[a_1,\cdots,a_{m+1}]$ must be in the null space of the matrix
\begin{equation*}
    \begin{bmatrix}
        P_1(t_1)&\cdots&P_1(t_{m+1})\\
        \vdots & &\vdots\\
        P_m(t_1)&\cdots&P_m(t_{m+1})
    \end{bmatrix}\in\Real^{m\times (m+1)}.
\end{equation*}
Therefore, $[a_1,\cdots,a_{m+1}]$ must be one-dimensional. 

We now define the vector-valued function
$$\vec{P}_j=[P_j, \tilde{D}^{(1)}P_j,\cdots,\tilde{D}^{(m-1)}P_j]^\top .$$
Because each $\tilde{D}^{(j)}$ is linear differential operator, we have
\begin{equation}
\label{eq:proof_theorem_right_KP_1}
    a_1\vec{P}_j(t_1)+\cdots +a_{m+1}\vec{P}_j(t_{m+1})=0
\end{equation}
for all $j=1,\cdots,m$.
On the other hand, we can derive  from \eqref{eq:first_order_Pi} that, for any $t_0\leq \mu\leq t \leq T$, $\vec{P}_j(t)$ can be solved via initial condition $\vec{P}_j(s)$:
\begin{equation}
    \label{eq:proof_theorem_right_KP_2}\vec{P}_j(t)=e^{\int_\mu^T F^*(\tau)d\tau}\vec{P}_j(\mu).
\end{equation}
Because \eqref{eq:proof_theorem_right_KP_2} holds true for any $\mu\leq t$. By combining \eqref{eq:proof_theorem_right_KP_1} and \eqref{eq:proof_theorem_right_KP_2}, we can derive
\begin{equation}
\label{eq:proof_theorem_right_KP_3}
    a_1e^{\int_\mu^{t_1}F^*(\tau)d\tau}+\cdots +a_{m+1}e^{\int_\mu^{t_{m+1}}F^*(\tau)d\tau}=0
\end{equation}
for any $\mu\leq t_1$. Let $\mu=t_1$ and multiply both sides of \eqref{eq:proof_theorem_right_KP_3} by $R^*(t_1,t_1):=\Pi(t_1)$, we  have the desired result 
\[a_1e^{\int_{t_1}^{t_1}F^*(\tau)d\tau}\Pi(t_1)+\cdots +a_{m+1}e^{\int_{t_1}^{t_{m+1}}F^*(\tau)d\tau}\Pi(t_1)=\sum_{j=1}^{m+1}a_jR^*(t_1,t_j)=0.\]

(2) Suppose $\sum_{j=1}^{m+1}a_jR^*(t_1,t_j)=0$. From $R^*(s,t)=e^{\int_s^T F^*(\tau)d\tau}\Pi(s)$, we can have
\begin{equation}
\label{eq:proof_theorem_right_KP_4}
\sum_{j=1}^{m+1}a_je^{\int_{t_1}^{t_j}F^*(\tau)d\tau}\Pi(t_1)=0.
\end{equation}
Multiply both sides of \eqref{eq:proof_theorem_right_KP_4} by $[\Pi(t_1)]^{-1}[\vec{P}_1(t_1),\cdots,\vec{P}_m(t_1)]^\top $ (the invertibility of $\Pi(t_1)$ will be proved in  Lemma \ref{lem:Pi_invertible}) , together with \eqref{eq:proof_theorem_right_KP_2}, we can have the desired result:
$$\sum_{j=1}^{m+1}a_j \left[P_1(t_j),\cdots, P_m(t_j)\right]^\top =0.$$
\end{proof}

\subsection{Backward stochastic differential equation}
To construct the left-KP, it is natural to consider the backward version of \eqref{eq:product_form}, because  time of  backward SDE  runs in a reversed direction $\tau=-t$. By doing so, we obtain another set of $m$ fundamental solutions, and consequently, the minimum  $s$ for constructing the left-KP is also $s=m+1$.  We first need the following lemma for the existence of the target backward SDE :
\begin{lemma}
    \label{lem:Pi_invertible}
    Suppose  Condition \ref{condition:full_rank} holds. Then $\Pi(t)=R(t,t)$ is invertible for any $t\in(t,T)$.
\end{lemma}
\begin{proof}[Proof of Lemma \ref{lem:Pi_invertible}]
We prove by induction. For the base case $m=1$, it is clear that $\Pi(t)=u(t)$ is invertible for any $0<|u(t)|<\infty$. Suppose the lemma holds for $m-1$, then for the case $m$, we now discuss how the determinant of $\Pi(t)$ changes with $t$.  

     From (6.2) in \cite{sarkka2019applied}, $\Pi$  satisfies the following differential equation:
    \[\partial_t\Pi(t)=F^*(t)\Pi(t)+\Pi(t)F^*(t)^\top +LL^\top.\]
    We then apply Jacobi's formula on the determinant of $\Pi$, then 
    \begin{align*}
        \partial_t\text{det}[\Pi(t)]=&\text{Tr}\left\{\text{adj}[\Pi(t)]\partial_t\Pi(t)\right\}\\
        =&\text{Tr}\left\{\text{adj}[\Pi(t)]F^*(t)\Pi(t)+\text{adj}[\Pi(t)]\Pi(t)F^*(t)^\top +\text{adj}[\Pi(t)] LL^\top\right\}\\
        =& \underbrace{\text{Tr}\left\{\text{adj}[\Pi(t)]\Pi(t)\left[F^*(t)+F^*(t)^\top \right]\right\}}_A+\underbrace{\text{Tr}\left\{ LL^\top\text{adj}[\Pi(t)]\right\}}_B
    \end{align*}
    where $\text{adj}[A]$ denote the adjugate of a matrix $A$ and the last line is from the properties $\text{Tr}[AB]=\text{Tr}[BA]$ and $\text{Tr}[A+B]=\text{Tr}[A]+\text{Tr}[B]$.

    For term $A$, we have $A=0$ regardless if $\text{det}[\Pi(t)]=0$,  because the trace of $F^*(t)+F^*(t)^\top $ is zero and $\text{adj}[\Pi(t)]\Pi(t)=\text{det}[\Pi(t)]\mathrm{I}$. 
    
    For term $B$, because $L=[0,\cdots,0,1]^\top $, we can have the following identity via direct calculations
    \[\text{Tr}\left\{ LL^\top\text{adj}[\Pi(t)]\right\}=\text{det}[\Pi_{1:(m-1),1:(m-1)}].\]

    Determinant of  $\Pi_{1:(m-1),1:(m-1)}$ must be non-negative because it is a covariance matrix. If $\text{det}[\Pi_{1:(m-1),1:(m-1)}]=0$ , it simply means that there exists non-zero $\{\alpha_j\}_{j=0}^{m-2}$ such that
    \begin{align}
        \sum_{j=0}^{m-2}\alpha_j \tilde{D}^{(j-1)}y(t)=0.\label{eq:Pi_invertible_proof_1}
    \end{align}
    Take the time derivative on both sides of \eqref{eq:Pi_invertible_proof_1}, we have
    \begin{align}
        \sum_{j=0}^{m-2}\alpha_j \hat{D}^{(j)}y(t)=0.\label{eq:Pi_invertible_proof_2}
    \end{align}
    where $\hat{D}^{(j)}=\partial_t\tilde{D}^{(j-1)}$ is an order-$j$ linear differential operator. This reduces to the following case for $m-1$ 
    \begin{align}
        \alpha_0\partial_ty+\alpha_1\partial_t\frac{1}{u_2}\partial_t y+\alpha_2\partial_t\frac{1}{u_3}\partial_t\frac{1}{u_2}\partial_t y+\cdots+\alpha_{m-2}\left(\partial_t\frac{1}{u_{m-1}}\partial_t\cdots\frac{1}{u_2}\partial_t y\right)=0.\label{eq:Pi_invertible_proof_3}
    \end{align}
    However, under Condition \ref{condition:full_rank}, all $u_i$ are bounded away from $0$ and infinity. So \eqref{eq:Pi_invertible_proof_3} contradicts with our induction assumption that for $m-1$, the determinant of $\Pi(t)$ is non-zero. We must have 
    \[\partial_t\text{det}[\Pi(t)]=\text{det}[\Pi_{1:(m-1),1:(m-1)}]>0.\]
    Therefore, $\text{det}[\Pi(t)]>0$ for any $t>t_0$.
\end{proof}

We now can apply Lemma 1 of \cite{ljung1976backwards} to obtain the backward version of the first-order Markov model \eqref{eq:forward_Markov} with the same covariance matrix $R^*(s,t)$:
\begin{theorem}[\citeauthor{ljung1976backwards}]
\label{thm:backward_Markov}
    The backward equation of \eqref{eq:forward_Markov} is :
    \begin{equation}\label{eq:Markov_no_prod}
\left\{
\begin{array}{ll}
\partial_\tau x(\tau )=\left[F^*(\tau)+C(\tau)\right]x(\tau)-L W(T-\tau)  \\[0.5ex]
y(\tau)=Hx(\tau)
\end{array}
\right.,\quad \tau\in(t_0,T)
\end{equation}
where
\begin{equation*}
    C(\tau)=LL^\top\Pi^{-1}(\tau)=\begin{bmatrix}
        0 & 0 &\cdots &0\\
        \vdots&\vdots&\cdots&\vdots\\
        0 & 0 &\cdots &0\\
        C_1(\tau) &C_2(\tau) &\cdots&C_m(\tau)
    \end{bmatrix},
\end{equation*}
 and $\mathbb{E}[x(\tau)x(t)^\top ]=R^*(\tau,t)=\mathbb{E}[z(\tau)z(t)^\top ]$ for all $\tau,t\in[t_0,T]$. Moreover, for any $t_0\leq \tau \leq s\leq T$, the covariance matrix $R^*(t,s)$ satisfies
 \begin{equation}
     \label{eq:covariance_R}
     \partial_\tau R^*(\tau,\mu)=\left[F^*(\tau)+C(\tau)\right]R^*(\tau,\mu)\quad \Rightarrow\quad R^*(\tau,\mu)=e^{-\int_{\tau}^\mu F^*(\gamma)+C(\gamma)d\gamma}\Pi(\mu).
 \end{equation}
\end{theorem}
 Because both the GP $y\sim\CalN(0,K)$ and the covariance matrix $R^*$ have no preference to any direction of time, so they can also be represented by the backward SDE as follows
\begin{align}
    &\CalL_b[y](\tau)=\CalL[y](\tau)-\sum_{i=1}^{m}C_i(\tau)\tilde{D}_{\tau}^{(i-1)}y(\tau)=W(T-\tau),\nonumber\\
    &\CalL_b[R^*_{1,j}(\cdot,\mu)](\tau)=\mathbb{E}\left[\left\{\CalL_b[y](\tau)\right\}x_j(\mu)\right]=\mathbb{E}[W(T-\tau)x_j(\mu)]=0 \label{eq:backward_R}
\end{align}
 for any $j=1,\cdots,m$ and $t_0<\tau < \mu < T$, where $\tilde{D}^{0}_\tau=1$, $\tilde{D}_\tau^{(i)}=\prod_{j=2}^{i+1}\tilde{D}_j$ are $i$-th order differential operator acting on $\tau$.  The last equality of \eqref{eq:backward_R} is because the process $x$ runs in a backward direction, meaning that the white noise ahead of $x(\mu)$ in this direction is independent of it.

A direct consequence of \eqref{eq:backward_R} is that the set $\{R^*_{1,j}(\cdot,\mu)\}_{j=1}^m$ are the fundamental solutions of the operator $\CalL_b$ on the interval $(t_0, \mu)$.  Even though $\{R^*_{1,j}(\cdot,\mu)\}_{j=1}^m$ do not constitute the exact fundamental solutions like ${P_i}$, they are sufficient for constructing the left-KP equations:
\begin{theorem}
\label{thm:fundamental-Covariance_leftKP}
    Suppose Condition \ref{condition:full_rank} holds. For any consecutive points $t_1<\cdots<t_{m+1}$ in $(t_0,T)$, 
    \[\sum_{j=1}^{m+1}a_jR^*_{1,:}(t_j,t_{m+1})=0\quad\text{if and only if}\quad \sum_{j=1}^{m+1}a_j R^*(t_j,t_{m+1})=0 \]
   where $[a_j]_{j=1}^{m+1}$ is one-dimensional. Therefore, $s=m+1$ for an irreducible  left-sided KP.
\end{theorem}
 \begin{proof}[proof of Theorem \ref{thm:fundamental-Covariance_leftKP}]
 The if and only if part can be easily derived from the linearity of the differential operators $\tilde{D}_t^{(i-1)}$ and the fact that $R^*_{i,j}(t,t_{m+1})=\tilde{D}_t^{(i-1)}R^*_{1,j}(t,t_{m+1})$. 

 To show that ${a_j}$ is one dimensional, we only need to show that $\{R^*_{1,j}(\cdot,t_{m+1})\}$ are linearly independent fundamental solutions. We can notice that the covariance matrix $R^*(\cdot,t_{m+1})$ is, in fact, a generalized Wronskian of functions $\{R^*_{1,j}(\cdot,t_{m+1})\}$:
 \[R^*_{i,j}(t,t_{m+1})=\mathbb{E}\left[\left(\tilde{D}_t^{(i-1)}y(t)\right)z_j(t_{m+1})\right]=\tilde{D}_t^{(i-1)}R^*_{1,j}(t,t_{m+1}).\]
 Therefore, we only need to prove that  the determinant of $R^*(t,t_{m+1})$ is non-zero. From (11) in \cite{ljung1976backwards}, we have
 \[R^*(t,t_{m+1})=e^{\int_t^{t_{m+1}}F^*(\tau)+C(\tau)d\tau}\Pi(t_{m+1}).\]
 Recall from Lemma \ref{lem:Pi_invertible}  that $\Pi(t)$ is invertible. So $\text{det}[\Pi(t_{m+1})]>0$ and hence
 \[\text{det}[R^*(t,t_{m+1})]=\text{det}[e^{\int_t^{t_{m+1}}F^*(\tau)+C(\tau)d\tau}]\text{det}[\Pi(t_{m+1})]>0.\]
 \end{proof}

\subsection{Proof of the Main Theorem}
\label{sec:additive_kernel}
The fundamental solutions for the operators $\CalL$, and  $\CalL_b$ are also mutually independent:
 \begin{lemma}
 \label{lem:Pi_R_i_orthogonal}
    Let $\mathcal{{P}}=\text{span}\{P_j:j=1,\cdots,m\}$ and $\mathcal{R}=\text{span}\{R^*_{1,j}(\cdot,\mu):j=1,\cdots,m\}$. Then 
    the dimension of $\mathcal{P}\cup\CalK$ is $2m$, i.e., $\{P_j, R^*_{1,j}(\cdot,\mu):j=1,\cdots,m\}$ are linearly independent.
    
 \end{lemma}
\begin{proof}[Proof of Lemma \ref{lem:Pi_R_i_orthogonal}]
From condition \ref{condition:full_rank} and the invertibility of covariance $R$, it is obvious that dim($\mathcal{P}$)=dim$(\CalK)$=m. We only need to prove that $\mathcal{P}\perp \CalK$. Note that $R_{1,j}(\cdot,\mu)=\tilde{D}^{(j-1)}_\mu K(\cdot,\mu)$ where $K$ is the kernel function of $y$. Let $G(t,\mu)$ be the Green's function of $\CalL$, i.e., $\CalL[G(\cdot,\mu)](t)=\delta_{t-\mu}$. Then it is straightforward to derive that $K(t,\mu)=\int_{t_0}^TG(t,\tau)G(\tau,\mu)d\tau$ and, as a result, we have $R_{1,j}(t,\mu)=\int_{t_0}^TG(t,\tau)\tilde{D}^{(j-1)}_\mu G(\tau,\mu)d\tau$. Therefore, $\CalL R_{1,j}(\cdot,\mu)\neq 0$. On the other hand, any $P_j\in\mathcal{P}$ is the fundamental solution of $\CalL$, i.e., $\CalL P_j=0$. So we must have $P_j\perp \CalK$ for any $j$.
 \end{proof}
 
 From Theorem \ref{thm:fundamental-Covariance_rightKP}, Theorem \ref{thm:fundamental-Covariance_leftKP}, and Lemma \ref{lem:Pi_R_i_orthogonal}, we conclude that $s=2m+1$ for an irreducible KP; otherwise,  $\{P_j, R^*_{1,j}(\cdot,\mu):j=1,\cdots,m\}$ would not remain linearly independent. Now we can prove Theorem \ref{thm:fundamental-Covariance_centalKP}:

 \begin{proof}\textbf{of Theorem  \ref{thm:fundamental-Covariance_centalKP}:}
Define $h_j=P_j$ and $h_{j+m}=R_{1,j}(\cdot,t_{2m+1})$ for $j=1,\cdots,m$.  Let $h=[h_1,\cdots,h_{2m}]^\top $. Suppose Condition \ref{condition:full_rank} holds,  then  Theorem \ref{thm:fundamental-Covariance_rightKP} and \ref{thm:fundamental-Covariance_leftKP} imply that we must have a linear combination that satisfies both right and left-sided KPs simultaneously:
    \[\sum_{j=1}^{2m+1}a_jh(t_j)=0\quad\text{if and only if}\quad \sum_{j=1}^{2m+1}a_j R(t_j,t_{2m+1})=\sum_{j=1}^{2m+1}a_j R(t_1,t_{j})=0 \]
    where $\{a_j\}_{j=1}^{2m+1}$ is one-dimensional. From Lemma \ref{lem:Pi_R_i_orthogonal} ,  $\{h_j\}_{j=1}^{2m}$ are linear independent. Therefore, the null space of the $2m$-by-$(2m+1)$ matrix $[h(t_1),\cdots,h(t_{2m+1})]$ is one-dimensional.     If $s<2m+1$, then $\{h_j\}_{j=1}^{2m}$ are not linearly independent, which is a contradiction.
\end{proof}
Theorem \ref{thm:fundamental-Covariance_centalKP}  shows that for an irreducible KP, it is necessary and sufficient to consider $s=2m+1$ equations, which is a key ingredient in proving the Main Theorem \ref{thm:main}. \\[0.5ex] 

\begin{proof}\textbf{of  Theorem \ref{thm:main}}: We prove the theorem for central KPs; the same argument applies to the left- and right-sided cases. In  Algorithm \ref{alg:KP}, each central KP $\phi_i^{(j)}$ are linear combination of $s=2m+1$ covariance vectors $\{R_1(\cdot,t_j)\}_{j=i-m}^{i+m}$, so if they are KPs, then each of them is irreducible and they form a minimal KP system. Therefore, we only need to show that for any $m<i\leq n-m$, if  $\sum_{l=i-m}^{i+m}a_l[R_1(t_{i-m},t_l)\ R_{1}(t_{i+m},t_l)]=0$ holds, then the following equaitons should hold
\begin{eqnarray}
\label{eq:pf_themreom_main_1}
     \sum_{l=i-m}^{i+m} a_l R(t_{i-m},t_l)=0\quad \text{and}\quad  \sum_{l=i-m}^{i+m} a_l R(t_{i+m},t_l)=0.
\end{eqnarray}
   According  to definition, 
 \begin{equation}
     \label{eq:pf_themreom_main_2}
     \begin{aligned}
         \partial_tR_1^\top(t_{i-m},t)=&\E[z_1(t_{i-m})\partial_t z(t)]
         =\E[z_1(t_{i-m})F(t) z(t)]+\E[z_1(t_{i-m})LW(t)].
     \end{aligned}
 \end{equation}
Note that $W(t)$ is independent of $z_1(t_{i-m})$ so $\partial_tR_1^\top(t_{i-m},t)=F(t)R_1^\top(t_{i-m},t)$, which yields 
\begin{equation}
     \label{eq:pf_themreom_main_3}
     \begin{aligned}
         R_1^\top(t_{i-m},t)=\exp\{\int_{t_{i-m}}^tF(\mu)d\mu\}R_1^\top(t_{i-m},t_{i-m}).
     \end{aligned}
 \end{equation}
 Because the GP $y$ satisfies condition \ref{condition:full_rank}, its SS representation \eqref{eq:forward_Markov} is reversible, so the general equivalent SS representation is also reversible. According to Lemma 1 of \cite{ljung1976backwards}, its backward SS model with the same covariance matrix $R$ is in the form
  \begin{equation}\label{eq:pf_themreom_main_4}
-\partial_\tau x(\tau )=F_b(\tau)x(\tau)-L W(T-\tau),\quad 
y(\tau)=Hx(\tau)
,\quad \tau\in(t_0,T),
\end{equation}
for some $F_b(\tau)$ (the closed-form expression of $F_b$ was derived in \cite{ljung1976backwards}; however, it is not needed in our proof). Then, we can use the same reasoning as \eqref{eq:pf_themreom_main_2} to have
\begin{equation}
    \label{eq:pf_themreom_main_5}
\begin{aligned}
    \partial_\tau R_1^\top(t_{i+m},\tau)=-F_b(\tau)R^\top_1(t_{i+m},\tau).
\end{aligned}
\end{equation}
This yields
\begin{equation}
    \label{eq:pf_themreom_main_6}
\begin{aligned}
    R_1^\top (t_{i+m},\tau)=\exp\{-\int_\tau^{t_{i+m}}F_b(\mu)d\mu\}R_1^\top(t_{i+m},t_{i+m}).
\end{aligned}
\end{equation}

Substitute \eqref{eq:pf_themreom_main_3} and \eqref{eq:pf_themreom_main_6}, into $\sum_{l=i-m}^{i+m}a_l[R_1(t_{i-m},t_l)\ R_{1}(t_{i+m},t_l)]=0$, we then have
\begin{equation}
    \label{eq:pf_themreom_main_7}
    \begin{aligned}
        &\sum_{l=i-m}^{i+m}a_lR_1^\top(t_{i-m},t_l)=\sum_{l=i-m}^{i+m}a_l\exp\{\int_{t_{i-m}}^{t_l}F(\mu)d\mu\}R^\top_1(t_{i-m},t_{i-m})=0,\\
        &\sum_{l=i-m}^{i+m}a_lR^\top_1(t_{i+m},t_{l})=\sum_{l=i-m}^{i+m}a_l\exp\{-\int_{t_l}^{t_{i+m}}F_b(\mu)d\mu\}]R^\top_1(t_{i+m},t_{i+m})=0.
    \end{aligned}
\end{equation}
Note that $\{\exp\{\int_{t_{i-m}}^{t_l}F(\mu)d\mu\}\}_{l=i-m}^{i+m}$ and $\{\exp\{-\int_{t_l}^{t_{i+m}}F_b(\mu)d\mu\}_{l=i-m}^{i+m}$ are linearly independent full-rank matrices because they are exponential of matrix integrals. So \eqref{eq:pf_themreom_main_7} implies that
\begin{equation}
    \label{eq:pf_themreom_main_8}
    \begin{aligned}
        &\sum_{l=i-m}^{i+m}a_l\exp\{\int_{t_{i-m}}^{t_l}F(\mu)d\mu\}=0=
        \sum_{l=i-m}^{i+m}a_l\exp\{-\int_{t_l}^{t_{i+m}}F_b(\mu)d\mu\}].
    \end{aligned}
\end{equation}
So we can use the same reasoning as \eqref{eq:pf_themreom_main_3} and \eqref{eq:pf_themreom_main_6} to have
\begin{equation}
    \label{eq:pf_themreom_main_9}
    \begin{aligned}
        &\sum_{l=i-m}^{i+m}a_l\exp\{\int_{t_{i-m}}^{t_l}F(\mu)d\mu\}R^\top(t_{i-m},t_{i-m})=\sum_{l=i-m}^{i+m}a_lR^\top(t_{i-m},t_l)=0\\
        &\sum_{l=i-m}^{i+m}a_l\exp\{-\int_{t_l}^{t_{i+m}}F_b(\mu)d\mu\}R^\top(t_{i+m},t_{i+m})]=\sum_{l=i-m}^{i+m}a_lR^\top (t_{i+m},t_{l})=0.
    \end{aligned}
\end{equation}
Then from Theorem \ref{thm:existence}, we can have the desired result.
\end{proof}

\section{Technical Details for Kernel Packets for More Kernels}
\label{sec:proof of gkp}
\begin{proof} \textbf{of Theorem \ref{thm:combine_ker_add}:} 
Let $h=[h_1,\cdots,h_{2m}]^\top $ and $g=[g_1,\cdots,g_{2m}]^\top $ where
$h_j=R^{(1)}_{1,j}(t_1,\cdot),\quad h_{j+m}=R^{(1)}_{1,j}(t_{s+1},\cdot),\quad g_j=R^{(2)}_{1,j}(t_1,\cdot),\quad g_{j+m}=R^{(2)}_{1,j}(t_{s+1},\cdot)$ for $j=1,\cdots,m$.

We first show that $\eqref{eq:combine_KP_add}$ is a kernel packet. Because $\{\psi_i\}$ is the minimal span of the function space $\CalH=\text{span}\{h_i ,g_i:i=1,\cdots,2m\}$,  solution to the following linear system
\begin{equation}
    \label{eq:minimal_KP_additive}\sum_{j=1}^{s+1}a_j\psi(t_j)=0
\end{equation}
is one-dimensional and also solve the following two linear systems
\begin{align*}
    \sum_{j=1}^{s+1}a_jh(t_j)=0,\quad \sum_{j=1}^{s+1}a_jg(t_j)=0.
\end{align*}
From  Theorem \ref{thm:existence} and \eqref{eq:pf_themreom_main_1}, we can have
\begin{align*}
    \sum_{j=1}^{s+1}a_jR^{(1)}(t,t_j)=0,\quad \sum_{j=1}^{s+1}a_jR^{(2)}(t,t_j)=0\quad,\forall\ t\not\in(t_1,t_{s+1}).
\end{align*}
According to definition,  $s$ is the minimal number such that \eqref{eq:minimal_KP_additive} holds, so the KP \eqref{eq:combine_KP_add} is irreducible.
\end{proof}

\begin{proof} \textbf{of Theorem \ref{thm:combine_ker_multi}:}
From the definition of the minimal spanning set, we know if $\sum_{j=1}^{s+1}a_j\phi(t_j)=0$ then
\begin{equation}
    \label{eq:combine_ker_multi_pf_1}
    \sum_{j=1}^{s+1}a_j\left[R^{(1)}_1(t_1,t_j)\bigotimes R^{(2)}_1(t_1,t_j)\right]=\sum_{j=1}^{s+1}a_j\left[R^{(1)}_1(t_{s+1},t_j)\bigotimes R^{(2)}_1(t_{s+1},t_j)\right]=0.
\end{equation}
Note that $R^{(l)}_1=R^{(l)}e_1$ for $l=1,2$, where $e_1=[1,0,\cdots,0]^\top$. From the identity $\left(\mathbf{M}_1\mathbf{A}_1\right)\bigotimes\left(\mathbf{M}_2\mathbf{A}_2\right)=\left(\mathbf{M}_1\bigotimes\mathbf{M}_2\right)\left(\mathbf{A}_1\bigotimes\mathbf{A}_2\right)$, \eqref{eq:combine_ker_multi_pf_1} can be written as:
\begin{equation}
    \label{eq:combine_ker_multi_pf_2}
    \begin{aligned}
        &\sum_{j=1}^{s+1}a_j\left[\left(R^{(1)}(t_{1},t_j)\bigotimes R^{(2)}(t_{1},t_j)\right)\left(e_1\bigotimes e_1\right)\right]\\
        =&\sum_{j=1}^{s+1}a_j\left[\left(R^{(1)}(t_{s+1},t_j)\bigotimes R^{(2)}(t_{s+1},t_j)\right)\left(e_1\bigotimes e_1\right)\right]=0.
    \end{aligned}
\end{equation}
Further note that  $e_1\bigotimes e_1=[1,0,\cdots,0]^\top$ and for $t\geq \mu$, the Kronecker product can be written as
\begin{equation}
    \label{eq:combine_ker_multi_pf_3}
    \begin{aligned}
        R^{(1)}(t,\mu)\bigotimes R^{(2)}(t,\mu)=&\left(e^{\int_{\mu}^tF_1(\tau)d\tau}R^{(1)}(\mu,\mu)\right)\bigotimes\left(e^{\int_{\mu}^tF_2(\tau)d\tau}R^{(2)}(\mu,\mu)\right)\\
        =& \left(e^{\int_{\mu}^tF_1(\tau)d\tau\bigotimes{\bold I}+{\bold I}\bigotimes \int_{\mu}^tF_2(\tau)d\tau}\right)\left(R^{(1)}(\mu,\mu)\bigotimes R^{(2)}(\mu,\mu)\right)\\
        =& \left(e^{\int_{\mu}^tF_1(\tau)\bigoplus F_2(\tau)d\tau}\right)\left(R^{(1)}(\mu,\mu)\bigotimes R^{(2)}(\mu,\mu)\right)
    \end{aligned},
\end{equation}
where $\mathbf{M}_1\bigoplus\mathbf{M}_2=\mathbf{M}_1\bigotimes{\bold I}+\bold{I}\bigotimes\mathbf{M}_2$ is the Kronecker sum of $\bold{M}_1$ and $\bold{M}_2$. So \eqref{eq:combine_ker_multi_pf_2} is the KP equations for the follows SS model
\begin{equation}
\label{eq:combine_ker_multi_pf_4}
\partial_t z(t)= \left(F_1(t)\bigoplus F_2(t)\right) z(t)+LW(t).
\end{equation}
The rank of matrix $F_1(t)\bigoplus F_2(t)$ may less than $m^2$ so the required $s$ for a irreducible KP for \eqref{eq:combine_ker_multi_pf_4} may less than $2m^2+1$. However, according to \cite{ljung1976backwards}, \eqref{eq:combine_ker_multi_pf_4} is reversible because the only condition for backward SS model exists is the invertibility of $\E z(t)z(t)$. Because both $R^{(1)}$ and $R^{(2)}$ are invertible, we have $[\E z(t) z(t)]^{-1}=[R^{(1)}(t,t)]^{-1}\bigotimes [R^{(2)}(t,t)]^{-1}$. Then according to the definition of $s$, \eqref{eq:combine_ker_multi_pf_2} is the equation for an irreducible KP of \eqref{eq:combine_ker_multi_pf_4}.
\end{proof}

\begin{proof}\textbf{of Theorem \ref{thm:combine_ker_add_multi_dim}:}
The proof can be done by induction on dimension $d$. For the base case $d=1$, it is obvious that KP exists and is irreducible for $s=2m+1$ because this is what we have done for our paper. Suppose we have $D$-dimensional KP function and KP equations, i.e., given function 
\[H=[\phi_{1,1}\ \cdots\ \phi_{m,1}\ \psi_{1,1}\ \cdots\ \psi_{m,1}\ \cdots \ \phi_{1,D}\ \cdots\ \phi_{m,D}\ \psi_{1,D}\ \cdots\ \psi_{m,D}\ ]^\top \in\Real^{2mD},\]
and any $2mD+1$ points  $\{\bold{t}_i\}_{i=1}^{2mD+1}$, we can solve the constants $(a_1,\cdots,a_{2mD+1})$ such that
\[\sum_{i=1}^{2mD+1}a_iH(\bold{t}_i)=0, \quad \sum_{i=1}^{2mD+1}a_iR(\bold{t},\bold{t}_i)=0,\quad \forall \bold{t}\in U.\]
Also, the KP equation is irreducible because we have $2md$ linear independent functions and $2md+1$ coefficients $a_i$.

Suppose for the $(D+1)$-th dimension GP $y^{(D+1)}$, we have the covariance functions $\{\phi_{i,D+1}=R_{1,i}^{(D+1)}(\underline{t}^{(D+1)},\cdot)\}_{i=1}^m$ and $\{\psi_{i,D+1}=R_{1,i}^{(D+1)}(\overline{t}^{(D+1)},\cdot)\}_{i=1}^m$.   Function $H_{\text new}$ becomes
\begin{align*}
    H_{\text new}=&[H^\top \ \phi_{1,D+1}\ \cdots\ \phi_{m,D+1}\ \psi_{1,D+1}\ \cdots \ \psi_{m,D+1}  ]^\top =[H;h].
\end{align*}

Now at any $s=2m(D+1)+1$ scattered $(D+1)$-dimensional points $\{(\bold{t}_i,\tau_i)\}_{i=1}^s$, we first separate the point set as follows:
\[(\BT_j,\boldsymbol{\tau}_j)=\{(\bold{t}_j,\tau_j),\cdots,(\bold{t}_{2mD+j},\tau_{2mD+j})\},\quad j=1,\cdots,2m+1.\]

For each $j$, if we solve
\begin{equation}
\label{eq:KP_system_D_dim}
    \sum_{i=1}^{2mD+1}b^{(j)}_iH([\BT_j]_i)=0.
\end{equation}
Then, obviously, we have a $d$-dimensional KP from the inductive assumption:
\begin{equation}
\label{eq:KP_D_dim}
    f_j(\Bt):=\sum_{i=1}^{2mD+1}b_i^{(j)}R(\bold{t},[\bold{T}_j]_i)=0,\quad \forall \bold{t}\in U_j
\end{equation}
where 
$U_j=\times_{d=1}^D\{(-\infty,\min_{j\leq i\leq 2mD+j}\{t_{i,d}\})\bigcup (\max_{j\leq i\leq 2mD+j}\{t_{i,d}\},\infty)\}$ and  $R=\sum_{d=1}^DR^{(d)}$.

Given $b^{(j)}_i$, we can solve the following system
\begin{align}
    \label{eq:KP_system_Dplus1_dim}\sum_{j=1}^{2m+1}c_j\left(\sum_{i=j}^{2mD+j}b_i^{(j)}h(\tau_i)\right)=0.
\end{align}
From our inductive assumption, \eqref{eq:KP_D_dim} is irreducible for all $j$, so KP functions $\{f_j\}_{j=1}^{2m+1}$ are linear independent. Also, because $\{R(\cdot,[\bold{T}_j]_i)=0\}_{i,j}$ are linear independent, we can concluded that the rank of $[b^{(j)}_i]_{i,j}\in\Real^{(2mD+1)\times (2m+1)}$ is $2m+1$. Therefore, terms in the parenthesis of \eqref{eq:KP_system_Dplus1_dim} are linear independent. So constants $c_1,\cdots,c_{2m+1}$ are one-dimensional because the values of $h$ is $2m$-dimensional and \eqref{eq:KP_system_Dplus1_dim} is a KP equations of $R^{(D+1)}$ . We then can have a 1-dimensional  KP:
\begin{equation}
    \label{eq:KP_Dplus1_dim}
    \sum_{j=1}^{2m+1}c_j\left(\sum_{i=j}^{2mD+j}b_i^{(j)}R^{(D+1)}(\tau,\tau_i)\right)=0,\quad \forall\tau\not\in (\min_i \tau_i,\max_i\tau_i).
\end{equation}

Now we can finish the proof by notice that we have solve the $D+1$ dimensional KP equations by putting \eqref{eq:KP_system_D_dim} and \eqref{eq:KP_system_Dplus1_dim} together:
\[\sum_{j=1}^{2m+1}c_j\left(\sum_{i=j}^{2mD+j}b_i^{(j)}H_{\text new}(\bold{t}_i,\tau_i)\right)=0=\sum_{i=1}^{2m(D+1)+1}\alpha_iH_{\text new}(\bold{t}_i,\tau_i).\]
Then, from \eqref{eq:KP_D_dim} and \eqref{eq:KP_Dplus1_dim}, we have the $(D+1)$-dimensional KP
\[\sum_{i=1}^{2m(D+1)+1}\alpha_i\left(R(\bold{t},\bold{t}_i)+R^{(D+1)}(\tau,\tau_i)\right)=\psi(\tau)+\sum_{j=1}^{2m+1}\phi_j(\bold{t})=0\quad \forall (\bold{t},\tau)\not\in U\]
where $U=(\bigcup_jU_j)\bigcup \{(-\infty, \min \tau_i)\bigcup (\max \tau_i,\infty)\}$.
\end{proof}

\begin{proof}\textbf{of Theorem \ref{thm:combine_ker_prod_multi_dim}:}
From the definition of $H$, we can note that $H:\Real^D\to \Real^{2m^D}$ is the Kronecker product of $\{R^{(d)}_1\}$:
\begin{equation}
    \label{eq:combine_ker_prod_multi_dim_1}
    H(\Bt)=\bigotimes_{d=1}^D \left[R_{1}^{(d)}(\underline{t}^{(d)},t_d),  R_{1}^{(d)}(\overline{t}^{(d)},t_d)\right].
\end{equation}
For $s=(2m)^D+1$ points $\{\Bt_i\}_{i=1}^s$ and coefficients $\{a_i\}_{i=1}^s$ such that $\sum_{i=1}^sa_i H(\Bt_i)=0$, because $\{R_{1}^{(d)}(\underline{t}^{(d)},t_d),  R_{1}^{(d)}(\overline{t}^{(d)},t_d)\}$ are linearly independent, $\{a_i\}_{i=1}^s$ are one dimensional and $\sum_{i=1}^s a_i R(\Bt,\Bt_i)$ is irreducible if we can prove that is is a KP.

From \cite{ljung1976backwards}, the GP $y^{(d)}$ has both a forward and backward SS model representations $z^{(d)}(t)$ and $x^{(d)}(\tau)$, respectively, with the same covariance $R^{(d)}$:
\begin{align*}
 &\partial_t z^{(d)}(t)=F^{(d)}(t) z^{(d)}(t)-LW(t), \quad y^{(d)}(t)=Hz(^{(d)}t),\quad  t\in (t_0,T),\\
 -&\partial_\tau x^{(d)}(\tau) F^{(d)}_bx^{(d)}(\tau)-LW(T-\tau), \quad y^{(d)}(\tau)=Hx^{(d)}(\tau),\quad  t\in (t_0,T),
\end{align*}
and the covariance $R^{(d)}$ can be represented by the forward process as
\begin{eqnarray}
\label{eq:combine_ker_prod_multi_dim_2}
    R^{(d)}(t,\tau)=
    \begin{cases}
        \exp\{\int_{\tau}^tF^{(d)}(u)du\} R^{(d)}(\tau,\tau), & t_0\leq \tau\leq t\leq T \\
        R^{(d)}(t,t)\exp\{\int^{\tau}_t [F^{(d)}(u)]^\top du\},& t_0\leq t\leq \tau\leq T
    \end{cases}.
\end{eqnarray}
and the backward process
\begin{eqnarray}
\label{eq:combine_ker_prod_multi_dim_3}
    R^{(d)}(t,\tau)=
    \begin{cases}
        \exp\{-\int_{\tau}^tF_b^{(d)}(u)du\} R^{(d)}(t,t), & t_0\leq \tau\leq t\leq T \\
         R^{(d)}(\tau,\tau)\exp\{-\int^{\tau}_t [F_b^{(d)}(u)]^\top du\},& t_0\leq t\leq \tau\leq T
    \end{cases}.
\end{eqnarray}

Substitute \eqref{eq:combine_ker_prod_multi_dim_2} and \eqref{eq:combine_ker_prod_multi_dim_3} into $\sum_{i=1}^sa_i H(\Bt_i)=0$, we can derive that
\begin{equation*}
    \begin{aligned}
        &\sum_{i=1}^s a_i \bigotimes_{d=1}^d\left[R^{(d)}_1(\underline{t}^{(d)},\underline{t}^{(d)})\exp\{\int_{\underline{t}^{(d)}}^{t_{i,d}}[F^{(d)}(u)]^\top du\},R^{(d)}_1(\overline{t}^{(d)},\overline{t}^{(d)})\exp\{-\int^{t_{i,d}}_{\overline{t}^{(d)}} [F_b^{(d)}(u)]^\top du\}  \right]=0.
    \end{aligned}
\end{equation*}
Because $\{\exp\{\int_{\underline{t}^{(d)}}^{t_{i,d}}F^{(d)}(u)du\}\}$ and $\{\exp\{-\int^{t_{i,d}}_{\overline{t}^{(d)}} F_b^{(d)}(u)du\}\}$ are full-rank matrices and linearly independent, we can use the same argument in the proof of Theorem \ref{thm:main} to have:
\begin{equation}
     \label{eq:combine_ker_prod_multi_dim_5}
     \sum_{i=1}^sa_i\bigotimes_{d=1}^D \left[R^{(d)}(\underline{t}^{(d)},t_{i,d}),\ R^{(d)}(\overline{t}^{(d)},t_{i,d})\right]=0.
\end{equation}

Denote $P^{(d)}(t)=\int_{t_0}^tF^{(d)}(u)du$ and $Q^{(d)}(t)=-\int_{t}^TF_b^{(d)}(u)du$ so we have the following identities for the transition matrix
\begin{align*}
    &\exp\{\int_{\tau}^tF^{(d)}(u)du\}=\exp\{P^{(d)}(t)-P^{(d)}(\tau)\},\\
    &\exp\{-\int_{t}^\tau F_b^{(d)}(u)du\}=\exp\{Q^{(d)}(t)-Q^{(d)}(\tau)\}.
\end{align*}
Note that
\begin{align*}
   P^{(d)}(t) P^{(d)}(\tau)-P^{(d)}(\tau)P^{(d)}(t)=&\int_{t_0}^t\int_{t_0}^\tau F(u_1)F(u_2)du_1du_2-\int_{t_0}^\tau \int_{t_0}^t  F(u_1)F(u_2)du_1du_2\\
   =&0.
\end{align*}
So matrices $P^{(d)}(t)$ and $P^{(d)}(\tau)$ commute. Similarly, we can also derive that $Q^{(d)}(t)$ and $Q^{(d)}(
\tau)$ commute. Therefore, we combine \eqref{eq:combine_ker_prod_multi_dim_2} and \eqref{eq:combine_ker_prod_multi_dim_3}  as:

\begin{eqnarray}
\label{eq:combine_ker_prod_multi_dim_6}
    R^{(d)}(t,\tau)=
    \begin{cases}
        \exp\{P^{(d)}(t)\}\exp\{-P^{(d)}(\tau)\} R^{(d)}(\tau,\tau), & t_0\leq \tau\leq t\leq T \\
        \exp\{Q^{(d)}(t)\}\exp\{-Q^{(d)}(\tau)\}  R^{(d)}(\tau,\tau),& t_0\leq t\leq \tau\leq T
    \end{cases}.   
\end{eqnarray}

 For $\Bd\in\{\pm 1\}^D$, define $\underline{U}^{(d)}=(-\infty, \underline{t}^{(d)})$ and $\overline{U}^{(d)}=(\overline{t}^{(d)},\infty)$, and
 \begin{equation*}
     U_{\Bd}=\cup_{d=1}^D U^{(d)},\quad \text{where}\ U^{(d)}=\begin{cases}
         \underline{U}^{(d)} &\text{if}\  d=-1\\
         \overline{U}^{(d)} &\text{if}\ d=1
     \end{cases}.
 \end{equation*}
So $U=\cup_{\Bd\in\{\pm 1\}^D}U_{\Bd}$. Given any $\Bd$, if its $d$-th entry $[\Bd]_d=1$, we select the first line representation in \eqref{eq:combine_ker_prod_multi_dim_6} for $R^{(d)}$ and if $[\Bd]_d=-1$, we select the second line representation in \eqref{eq:combine_ker_prod_multi_dim_6} for $R^{(d)}$. In either selection, we can get a representation of $R^{(d)}$ for any $t_d\in U^{(d)}$ as
\begin{equation}
    \label{eq:combine_ker_prod_multi_dim_7}
    R^{(d)}(t_d, t_{i,d})= \bold{S}^{(d)}_l(t_d) \bold{S}^{(d)}_r(t_{i,d})
\end{equation}
where
\begin{equation*}
    \bold{S}^{(d)}_l(t)=\begin{cases}
        \exp\{P^{(d)}(t)\},\  &\text{if}\ [\Bd]_d=1\\
        \exp\{Q^{(d)}(t)\},\  &\text{else}
    \end{cases},\ \bold{S}^{(d)}_r(t)=\begin{cases}
        \exp\{-P^{(d)}(t)\} R^{(d)}(t,t) &\text{if}\ [\Bd]_d=1\\
        \exp\{-Q^{(d)}(t)\}  R^{(d)}(t,t),\   &\text{else}
    \end{cases}.
\end{equation*}
Then for any $\Bd\in\{\pm 1\}^D$, define, 
\begin{equation*}
t^*_d=\begin{cases}
    \underline{t}^{(d)},\quad &\text{if}\ [\Bd]_d=-1\\
    \overline{t}^{(d)},\quad &\text{if}\ [\Bd]_d=1.
\end{cases}
\end{equation*}
S we can use \eqref{eq:combine_ker_prod_multi_dim_5} to derive that 
\begin{equation}
    \label{eq:combine_ker_prod_multi_dim_8}
    \begin{aligned}
    0=\sum_{i=1}^s a_i \bigotimes_{d=1}^DR(t^*_d,t_{i,d})= &\sum_{i=1}^s a_i \bigotimes_{d=1}^D\left[\bold{S}^{(d)}_l(t^*_d) \bold{S}^{(d)}_r(t_{i,d})\right]\\
    =& \sum_{i=1}^s a_i \left[\bigotimes_{d=1}^D\bold{S}^{(d)}_l(t^*_d)\right]\left[\bigotimes_{d=1}^D\bold{S}^{(d)}_r(t_{i,d})\right]\\
    =& \left[\bigotimes_{d=1}^D\bold{S}^{(d)}_l(t^*_d)\right]\sum_{i=1}^s a_i \left[\bigotimes_{d=1}^D\bold{S}^{(d)}_r(t_{i,d})\right].
\end{aligned}
\end{equation}
Again, from the fact that $\{\bold{S}^{(d)}_r(t_{i,d})\}$ are linearly independent full-rank matrices, we have:
\begin{equation}
\label{eq:combine_ker_prod_multi_dim_9}
    \sum_{i=1}^s a_i \left[\bigotimes_{d=1}^D\bold{S}^{(d)}_r(t_{i,d})\right]=0.
\end{equation}
Finally, for any $\Bd$ and any $\Bt=[t_1,\cdots,t_D]\in U_\Bd$, calculations similar to \eqref{eq:combine_ker_prod_multi_dim_8} yields
\begin{equation}
\label{eq:combine_ker_prod_multi_dim_10}
    \sum_{i=1}^s a_i \bigotimes_{d=1}^DR(t_d,t_{i,d})=\left[\bigotimes_{d=1}^D\bold{S}^{(d)}_l(t_d)\right]\sum_{i=1}^s a_i \left[\bigotimes_{d=1}^D\bold{S}^{(d)}_r(t_{i,d})\right]=0
\end{equation}
where the last equality is from \eqref{eq:combine_ker_prod_multi_dim_9}. Because \eqref{eq:combine_ker_prod_multi_dim_10} holds for any $\Bt\in \cup_{\Bd\in\{\pm 1\}^D}U_\Bd=U$, we can finish the proof.
\end{proof}


\section{From SS model to Kernel}
\label{sec:conversion}
 We  introduce the basic converting between SDE and kernel. For a detailed introduction, examples, and references of further advanced material, please refer to \cite{sarkka2019applied,solin2016stochastic,benavoli2016state}. The SDE representation \eqref{eq:SODE} of a GP can be reformulated as the first order Markov process \eqref{eq:forward_Markov}. Recall that \eqref{eq:forward_Markov} is written as follows: 

\begin{equation}\label{eq:forward_Markov_appendix_1}
\left\{
\begin{array}{ll}
dz(t)=F(t)z(t)dt+LW(t)  \\[0.5ex]
y(t)=Hz(t)
\end{array}
\right.,\quad t\in(t_0,T)
\end{equation}
where $W(t)$ is a  white noise process
$\E[W(t)W(t')]=\delta(t-t')$
with $\delta(t-t')$ being the Dirac delta. Here, the matrix function $F(t)$ is a general matrix function not necessary in the form given in \eqref{eq:forward_Markov} and $L$  is also a general vector function.

Our goal is then to solve the kernel function $K(t,t')=\E[y(t),y(t')])$, which  is simply
 \[K(t,t')=\E[y(t),y(t')]=H\E[z(t)z^\top (t')]H^\top.\]
We can note that \eqref{eq:forward_Markov_appendix_1} should be interpreted as
\begin{equation*}
z(t)=\int_{\tau}^\top \Phi(t,s)LW(s)ds+\Phi(t,\tau)z(\tau)
\end{equation*}
where matrix function $\Phi(t,s)$ is known as the transfer matrix and it satisfies the following ODE
\begin{equation}
\label{eq:forward_Markov_appendix_2}
    \partial_t\Phi(t,s)=F(t)\Phi(t,s),\quad \Phi(s,s)=\boldsymbol{I}.
\end{equation}
We can solve\eqref{eq:forward_Markov_appendix_2} directly to get $\Phi(t,s)=\exp(\int_s^T F(\tau)d\tau)$ for any $s\leq t$.
 It is then straightforward to calculate the kernel function of $y(t)$ as follows:
 \begin{equation}
 \label{eq:forward_Markov_appendix_3}
K(t,t') = H\Phi(t\vee t,t\wedge t')\Pi(t\wedge t')H^\top 
\end{equation}
where $t\vee t'$ denotes the max between $t$ and $t'$,  $t\wedge t'$ denotes the min between $t$ and $t'$, and $\Pi(t)=R(t,t)$ and obeys
\begin{equation}
    \label{eq:forward_Markov_appendix_4}\partial_t\Pi(t)=F(t)\Pi(t)+\Pi(t)F^\top (t)+LL^\top .
\end{equation}

Solution to the matrix equation \eqref{eq:forward_Markov_appendix_4} is given as:
\[\Pi(t)=\Phi(t,t_0)\Pi(t_0)\Phi^\top (t,t_0)+\int_{t_0}^\top \Phi(t,\tau)L[\Phi(t,\tau)L]^\top d\tau.\]

The equation \eqref{eq:forward_Markov_appendix_3} plays an important role in the conversion between SDE and kernel because it provides an explicit formula for calculating the kernel $K$ from the given SDE \eqref{eq:forward_Markov_appendix_1}.

\bibliography{GeneralKP}
\end{document}